\documentclass{article}

\usepackage{iclr2020_conference,times}
\iclrfinalcopy
\usepackage{algorithm} 
\usepackage{algorithmic}  
\usepackage[algo2e]{algorithm2e} 

\usepackage[utf8]{inputenc} % allow utf-8 input
\usepackage[T1]{fontenc}    % use 8-bit T1 fonts
\usepackage[hidelinks]{hyperref}       % hyperlinks
\usepackage{url}            % simple URL typesetting
\usepackage{booktabs}       % professional-quality tables
\usepackage{amsfonts}       % blackboard math symbols
\usepackage{nicefrac}       % compact symbols for 1/2, etc.
\usepackage{microtype}      % microtypography

\usepackage{amsmath}  % Math fonts
\usepackage{amsfonts} %
\usepackage{amssymb}  %
\usepackage{amsthm,bm}
\usepackage{mathtools}
\usepackage{dsfont}
\usepackage{multirow}
\usepackage{thm-restate}
\usepackage{graphicx}
\usepackage{todonotes}
\usepackage{diagbox}
\usepackage{comment}

\usepackage{wrapfig}
\usepackage{enumitem}

\usepackage{xspace} 
\newtheorem{theorem}{Theorem}
\newtheorem{lemma}{Lemma}
\newtheorem{remark}{Remark}

\newtheorem{definition}[theorem]{Definition}

\newtheorem{claim}[theorem]{Claim}

\newtheorem{example}{Example}

\DeclareMathOperator*{\E}{{\mathbb{E}}}

\DeclarePairedDelimiter\floor{\lfloor}{\rfloor}

\newcolumntype{?}{!{\vrule width 1pt}}

\newcommand{\paraspace}{\vspace{-2mm}}

\newcommand{\figref}[1]{Figure~\ref{fig:#1}}
\newcommand{\tabref}[1]{Table~\ref{tab:#1}}
\newcommand{\secref}[1]{Sec.~\ref{sec:#1}}
\newcommand{\thmref}[1]{Theorem~\ref{thm:#1}}

\newcommand{\appref}[1]{Appendix~\ref{sec:#1}}

\newcommand{\eq}[1]{\eqref{eq:#1}}

\newcommand{\bb}[0]{\bm{b}}
\newcommand{\sfd}[0]{\mathsf{d}}
\newcommand{\sfW}[0]{\mathsf{W}}
\newcommand{\kwta}[0]{$k$-WTA\xspace}

\newlength\savedwidth

\SetKwInput{KwInput}{Input}
\SetKwInput{KwOutput}{Output}

\title{Enhancing Adversarial Defense by $k$-Winners-Take-All}
\author{
 Chang Xiao \qquad Peilin Zhong\qquad Changxi Zheng\\
 \ Columbia University\\
  \texttt{\{chang, peilin, cxz\}@cs.columbia.edu} \\
}
\begin{document}

\maketitle
\vspace{-1mm}
\begin{abstract}
\vspace{-1mm}
We propose a simple change to existing neural network structures
for better defending against gradient-based adversarial attacks.
Instead of using popular activation functions (such as ReLU), we advocate the use
of $k$-Winners-Take-All (\kwta) activation, a $C^0$ discontinuous function that purposely 
invalidates the neural network model's gradient at densely distributed input data points.
The proposed \kwta activation can be readily used in nearly all existing networks and training methods
with no significant overhead.
Our proposal is theoretically rationalized.
We analyze why the discontinuities in \kwta networks can largely 
prevent gradient-based search of adversarial examples and 
why they at the same time remain innocuous to the network training.
% We propose a new method against gradient-based adversarial attack using the
% $k$-winners-take-all ($k$-WTA) activation function. We theoretically prove that
% the discontinuity behavior in $k$-WTA can largely increase the difficulty of
% finding adversarial samples using gradient-based method, on the other hand,
% this discontinuity will not harm the training process. 
This understanding is also empirically backed.
We test \kwta activation on various network structures optimized by a training method, be it adversarial training or not.
In all cases, %% and on multiple datasets,
the robustness of \kwta networks outperforms that of traditional networks
under white-box attacks. 
%
% On various network structures, when \kwta activation is applied,
% the robustness performance of these models 
% We compare the robustness performance of \kwta activation to 
% that on various network structures
% On various network structures
% 
% Even without notoriously expensive adversarial training, the robustness performance of our networks is comparable to 
% conventional ReLU networks optimized by adversarial training.
% Furthermore, after also optimized through adversarial training, 
% our networks outperform the state-of-the-art methods
% under white-box attacks on various datasets that we experimented with.
%  With the cost of
%  training a regular neural network, our method achieves comparable results to
%  networks optimized with adversarial training, which is 20x more slower than
%  training a regular network. In addition, we found that $k$-WTA combine with
%  adversarial training achieves the state of the art results on CIFAR-10 against
%  PGD white box attack, with an accuracy improvement from 56.6\% to 70.3\%.
\end{abstract}

\vspace{-0.05in}
\section{Introduction}
\vspace{-0.05in}
% ---> adversarial examples is a big problem
In the tremendous success of deep learning techniques, there is a grain of salt.
It has become well-known that deep neural networks can be easily fooled by
\emph{adversarial examples}~\citep{intriguing2014}.
Those deliberately crafted input samples can mislead the networks %%in an almost arbitrary way.
to produce an output drastically different from what we expect.
In many important applications, 
from face recognition authorization to autonomous cars,
this vulnerability gives rise to serious security concerns~\citep{barreno2010security, barreno2006can,sharif2016accessorize,thys2019fooling}.

% ---> constructing adversarial examples require gradient
Attacking the network is straightforward. Provided a labeled data item $(\bm{x},y)$,
the attacker finds a perturbation $\bm{x}'$ imperceptibly similar to $\bm{x}$
but misleading enough to cause the network to output a label different from $y$.
By far, the most effective way of finding such a perturbation (or adversarial example)
is by exploiting the gradient information of the network with respect to its input:
the gradient indicates how to perturb $\bm{x}$ to trigger the maximal change of $y$.

% ---> natural idea is to make the gradient nonexistent: namely highly discontinuous (discontinuity is dense)
The defense, however, is challenging.
%% Despite the alternating proposals between
Recent studies showed that adversarial examples always exist if one tends to pursue
a high classification accuracy---adversarial
robustness seems at odds with the accuracy~\citep{tsipras2018robustness,  shafahi2018are, robustcost2018,weng2018towards,zhang2019theoretically}.
This intrinsic difficulty of eliminating adversarial examples suggests %%the researcher to ponder 
an alternative path: \emph{can we design a network whose adversarial examples are evasive rather than eliminated?}
Indeed, along with this thought is a series of works using obfuscated gradients as a defense mechanism~\citep{obfuscated-gradients}.
Those methods hide the network's gradient information by artificially discretizing the input~\citep{thermo2018,lin2018defensive}
or introducing certain randomness to the input~\citep{xie2018mitigating, guo2018countering} or the network structure~\citep{s.2018stochastic, cohen2019certified}
(see more discussion in \secref{obf}).
Yet, the hidden gradient in those methods can still be approximated, and as recently pointed out by~\citet{obfuscated-gradients},
those methods remain vulnerable.

\paraspace
\paragraph{Technical contribution I).}
% ---> yet there seems to be a paradox. network must be trainable
Rather than obfuscating the network's gradient, we make the gradient \emph{undefined}. %% in the first place.
This is achieved by a simple change to the standard neural network structure:
we advocate the use of a \emph{$C^0$ discontinuous} activation function, namely the $k$-Winners-Take-All (\kwta) 
activation, to replace the popular activation functions
such as rectified linear units (ReLU).
This is the only change we propose to a deep neural network.
All other components (such as BatchNorm, convolution, and pooling) as well as the training methods remain unaltered.
%% As a result, 
With no significant overhead,
\kwta activation can be readily
used in nearly all existing networks and training methods. 

%% properties are made possible is by kWTA -> what is kWTA why is it continuous -> f(x,w)
%% what is kWTA -> hard to attack
\kwta activation takes as input the entire output of a layer, retains its $k$ largest
values and deactivates all others to zero.
As we will show in this paper, even an infinitesimal perturbation to the input may cause a complete change 
to the network neurons' activation pattern, thereby resulting in a large jump 
in the network's output.
This means that, mathematically, if we use $f(\bm{x};\bm{w})$ to denote a \kwta network taking 
an input $\bm{x}$ and parameterized by weights $\bm{w}$,
then the gradient $\nabla_{\bm{x}}f(\bm{x};\bm{w})$ at certain $\bm{x}$ is 
undefined---$f(\bm{x};\bm{w})$ is $C^0$ \emph{discontinuous}.

\begin{wrapfigure}[13]{r}{0.405\textwidth}
  \centering
  \vspace{-6mm}
  \includegraphics[width=0.405\textwidth]{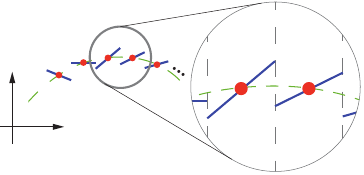}
  \vspace{-7mm}
  \caption{ \textbf{1D illustration.} 
  Fit a 1D function (green dotted curve) using a \kwta model provided
  with a set of points (red). The resulting model is piecewise continuous (blue curve),
  and the discontinuities can be dense.
\label{fig:illus}}
\end{wrapfigure}
\paraspace
\paragraph{Technical contribution II).}
More intriguing than the mere replacement of the activation function 
is \emph{why \kwta helps improve the adversarial robustness}.
We offer our theoretical reasoning of its behavior from two perspectives.
On the one hand, we show that the discontinuities of $f(\bm{x};\bm{w})$ is 
\emph{densely distributed} in the space of $\bm{x}$. Dense enough such that a tiny 
perturbation from any $\bm{x}$
almost always comes across some discontinuities, where the gradients 
are undefined and thus the attacker's search of adversarial examples becomes
blinded (see \figref{illus}). %% for a 1D illustration).

% --->
On the other hand, a paradox seemingly exists.
The discontinuities in the activation function 
also renders $f(\bm{x};\bm{w})$ discontinuous with respect to the network weights $\bm{w}$
(at certain $\bm{w}$ values).
But training the network relies on the presumption that the gradient with respect to 
the weights is almost always available.
Interestingly, we show that, under \kwta activation,
the discontinuities of $f(\bm{x};\bm{w})$ is rather \emph{sparse} in the space of $\bm{w}$,
intuitively because the dimension of $\bm{w}$ (in parameter space) is much larger than the dimension of $\bm{x}$ (in data space).
Thus, the network can be trained successfully.

\paraspace
\paragraph{Summary of results.}
%In addition to the theoretical analysis,
We conducted extensive experiments on multiple datasets under 
different network architectures, including ResNet~\citep{he2016deep}, DenseNet~\citep{huang2017densely}, 
and Wide ResNet~\citep{zagoruyko2016wide},
that are optimized by regular training as well as various adversarial training methods~\citep{madry2017towards,zhang2019theoretically,shafahi2019adversarial}.

In all these setups, we compare the robustness performance of 
using the proposed \kwta activation with commonly used ReLU activation
under several white-box attacks, 
including PGD~\citep{kurakin2016adversarial}, Deepfool~\citep{moosavi2016deepfool}, C\&W~\citep{carlini2017towards}, 
MIM~\citep{dong2018boosting}, and others.
In all tests, \kwta networks outperform ReLU networks.

The use of \kwta activation is motivated for defending against gradient-based adversarial attacks.
Our experiments suggest that the robustness improvement gained by simply switching to \kwta activation is universal,
not tied to specific network architectures or training methods.
To promote reproducible research, we will release our implementation of \kwta networks, 
along with our experiment code, configuration files and pre-trained models\footnote{https://github.com/a554b554/kWTA-Activation}.

\vspace{-0.03in}
\subsection{Related Work: Obfuscated Gradients as a Defense Mechanism}\label{sec:obf}
\vspace{-0.03in}
Before delving into \kwta details, we review prior adversarial defense methods that 
share the same philosophy with our method and highlight our advantages.
For a review of other attack and defense methods, we refer to \appref{related}. 

Methods aiming for concealing the gradient information from the attacker has
been termed as \emph{obfuscated gradients}~\citep{obfuscated-gradients} or
\emph{gradient masking}~\citep{papernot2017practical, tramer2017ensemble} techniques.
One type of such methods is by exploiting randomness, either 
randomly transforming the input before feeding it to the network~\citep{xie2018mitigating,guo2018countering}
or introducing stochastic layers in the network~\citep{s.2018stochastic}.
However, the gradient information in these methods can be estimated
by taking the average over multiple trials~\citep{obfuscated-gradients,athalye2017synthesizing}. As a result, 
these methods are vulnerable.

%  \textit{Obfuscated gradient} or \textit{gradient masking} refer to defenses
%  that does not provide useful gradient for attacker~\cite{obfuscated-gradients,
%  papernot2017practical, tramer2017ensemble}.
%  One major type of obfuscated gradient is \textit{stochastic gradient}, which
%  apply randomized transformations to the input before seeding it to the
%  model~\cite{xie2018mitigating,guo2018countering} or use stochastic layers in
%  the deep model~\cite{s.2018stochastic}, these randomness cause the
%  gradient-based attack fail due to incorrect gradient at test time.  However,
%  Athalye et al.~\cite{obfuscated-gradients} shows that defenses based on
%  stochastic gradient can be easily evaded by taking expectation over
%  transformation~\cite{athalye2017synthesizing}.

Another type of obfuscated gradient methods relies on the 
so-called \emph{shattered gradient}~\citep{obfuscated-gradients}, which aims to
make the network gradients nonexistent or incorrect to the attacker, by
purposely discretizing the input~\citep{thermo2018,ma2018characterizing} or
artificially raising numerical instability for gradient
evaluation~\citep{song2018pixeldefend,samangouei2018defensegan}. 
Unfortunately, these methods are also vulnerable.
As shown by~\citet{obfuscated-gradients}, they can be 
compromised by \textit{backward pass differentiable approximation} (BPDA). 
Suppose $f_i(\bm{x})$ is a non-differentiable component of 
a network expressed by $f=f_1\circ f_2\circ \cdots \circ f_n$.
The gradient $\nabla_{\bm{x}}f$ can be estimated as long as one can find a
smooth delegate $g$ that approximates well $f_i$ (i.e., $g(x) \approx f_i(x)$).

% Another type of obfuscated gradient is \textit{shattered gradient},
% which make the gradient nonexistent or incorrect, with respect to the input, by using nondifferentiable
% operators ~\cite{thermo2018,ma2018characterizing}, or adding numerical
% instability to the gradient ~\cite{song2018pixeldefend, samangouei2018defensegan}. 
% Shattered gradient can be caused either intentionally or unintentionally.
% To attack defenses where gradients are not readily
% available, Athalye et al. ~\cite{obfuscated-gradients} propose to use 
% \textit{backward pass differentiable approximation} (BPDA). 
% More specifically, assume $f_i(x)$ be the nondifferentiable layer in the deep neural network $f=f_1\circ f_2\circ \cdots \circ f_n$, 
% BPDA first find a differentiable approximation $g(x)$ that $g(x) \approx f_i(x)$.
% Then $\nabla_{x}f$ can be computed by performing the forward pass through $f_i(x)$,
% but on the backward pass use $g(x)$. 
% The authors shows that the above defense method based on shattered gradient are broken by the adversarial samples constructed by BPDA.

% 1. simple 2. dense 3. hard to approximate
In stark contrast to all those methods, 
a slight change of the \kwta activation pattern in an earlier layer of a network can cause a 
radical reorganization of activation patterns in later layers (as shown in \secref{theory}).
Thereby, \kwta activation
not just obfuscates the network's gradients but destroys them at certain input samples,
introducing discontinuities densely distributed in the input data space.
We are not aware of any possible smooth approximation of a \kwta network to launch BPDA attacks.
% To our knowledge, it is extremely hard, if not impossible, to effectively approximate
% a \kwta network using a smooth function.
Even if hypothetically there exists a smooth approximation of $k$-WTA activation, that approximation 
has to be applied to every layer. Then the network would accumulate the approximation error at each layer rapidly 
so that any gradient-estimation-based attack (such as BPDA) will be defeated.

\vspace{-0.05in}
\section{$k$-Winners-Take-All Activation}
\vspace{-0.05in}
% ---> kWTA date back early -> mention our made some extension
% ---> for boolean function
The debut of the \emph{Winner-Takes-All} (WTA) activation on the stage of neural networks 
dates back to 1980s, 
when ~\citet{grossberg1982contour} introduced shunting short-term memory equations in on-center off-surround networks
and showed the ability to identify the largest of $N$ real numbers.
Later, ~\citet{majani1989k} generalized the WTA network to identify the $K$ largest of $N$ real numbers,
and they termed the network as the K-Winners-Take-All (KWTA) network.
These early WTA-type activation functions output only boolean values, mainly
motivated by the properties of biological neural circuits. 
In particular,~\citet{maass2000computational,maass2000neural} has proved that any boolean
function can be computed by a single KWTA unit. 
Yet, the boolean nature of these activation functions differs starkly from the modern 
activation functions, including the one we use.

%  Unlike modern activation functions--and thus the one we will use---the early WTA-style activation functions
%  output only boolean values.

%  \textit{Winner-takes-all} (WTA) is biologically inspired computational process for neural networks
%  where neurons in a layer compete with each other for activation. 
%  A standard WTA only allows the neuron with the highest activation stays active while all other neurons shut down. In $k$-WTA, the top-$k$ activated neurons will stay active.

%  In the early age of neural networks, people have studied WTA scheme in various aspect.
%  Grossberg's shunting short-term memory equations~\cite{grossberg1982contour} is the first model to describe the WTA behavior in neural networks. 
%  Later, Majani et al.~\cite{majani1989k} generalize the WTA scheme to $k$-WTA.
%  Maass~\cite{maass2000neural,maass2000computational} further proves that any boolean function can be computed by a single $k$-WTA unit. 

% ---> useful in the context of robust learning

\vspace{-0.05in}
\subsection{Deep Neural Networks Activated by $k$-Winners-Take-All}\label{sec:kwta_sub}
\vspace{-0.04in}
We propose to use $k$-Winners-Take-All ($k$-WTA) activation, a natural generalization of the boolean
KWTA%
\footnote{In this paper, we use $k$-WTA to refer our activation function, while using
KWTA to refer the original boolean version by~\citet{majani1989k}.}~\citep{majani1989k}.
$k$-WTA retains the $k$ largest values of an $N\times1$ input vector and 
sets all others to be zero before feeding the vector to the next network layer, namely,
\begin{equation}\label{eq:kwta}
    \phi_k(\bm{y})_j=\left\{\begin{array}{ll}y_j, & y_j \in \{\text{$k$ largest elements of $\bm{y}$} \}, \\ 
0, & \text{Otherwise}.\end{array}\right.
\end{equation}
Here $\phi_k:\mathbb{R}^N \rightarrow \mathbb{R}^N$ is the $k$-WTA function
(parameterized by an integer $k$), $\bm{y} \in \mathbb{R}^N$ is the input to the activation, 
and $\phi_k(\bm{y})_j$ denote the $j$-the element of the output $\phi_k(\bm{y})$
(see the rightmost subfigure of \figref{act}).
Note that if $\bm{y}$ has multiple elements that are equally $k$-th largest, we break the tie by
retaining the element with smaller indices until the $k$ slots are taken.

% Consider the generalized form of $k$-WTA activation:
% \begin{align*}
% \phi_k(y)_j=\left\{\begin{array}{ll}y_j, & y_j \in \{\text{$k$ largest elements of $y$} \}, \\ 
% 0, & \text{Otherwise},\end{array}\right.
% \end{align*}
%  where $y \in \mathbb{R}^m$, $\phi_k:\mathbb{R}^m \rightarrow \mathbb{R}^m$, 
%  $\phi_k(y)_j$ represent the $j$-th element of $\phi_k(y)$.
%  A comparison between ReLU, max-pooling/maxout~\cite{goodfellow2013maxout}, LWTA and $k$-WTA can be found in \figref{act}.

\begin{figure*}[t]
	\centering
	%\vspace{-3mm}
	\includegraphics[width=0.97\textwidth]{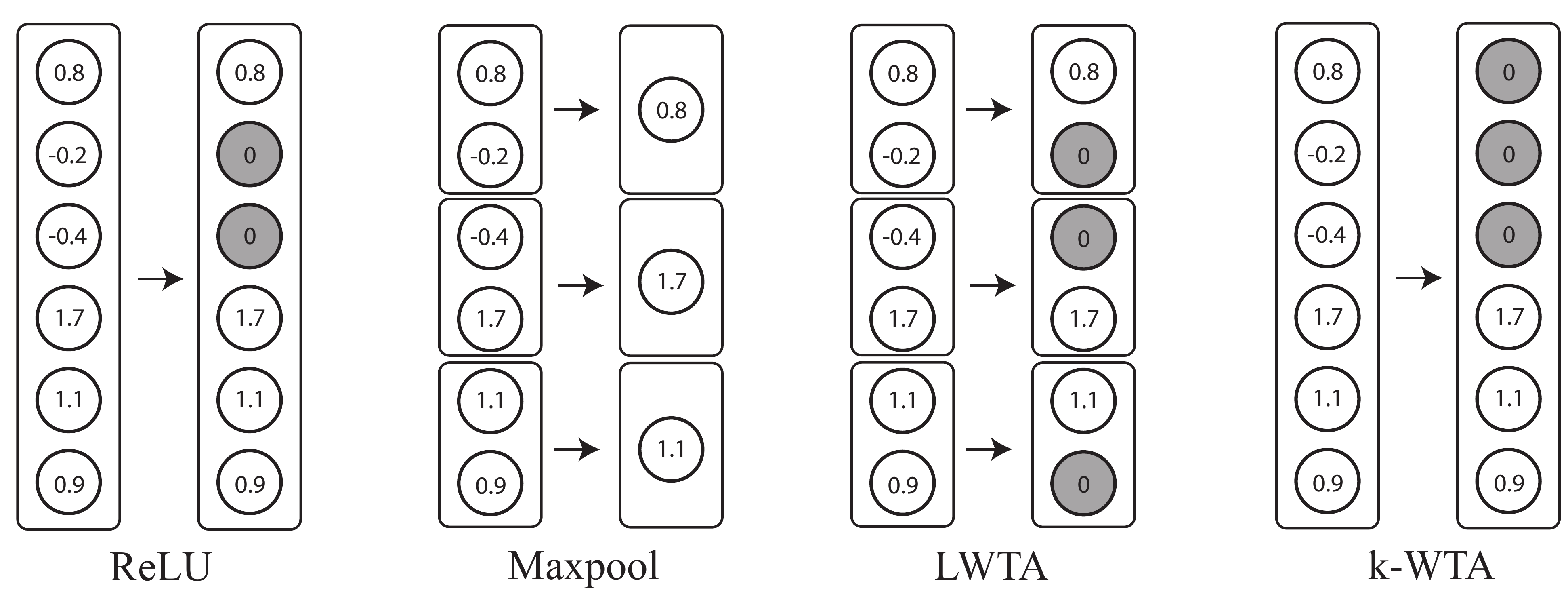}
	\vspace{-4.5mm}
	\caption{
            \textbf{Different activation functions.} 
	\textbf{ReLU}: all neurons with negative activation values will be set to zero.
	\textbf{Max-pooling}: only the largest activation in each group 
        is transmitted to the next layer, and this effectively downsample the output.
	\textbf{LWTA}: the largest activation in each group retains its value when entering the next layer, 
	others are set to zero.
	\textbf{$k$-WTA}: the $k$ largest activations in the entire layer 
        retain their values when entering the next layer,
	others are set to zero ($k=3$ in this example).
	Note that the output is not downsampled through ReLU, LWTA and $k$-WTA.
    \label{fig:act}}
	\vspace{-4mm}
\end{figure*}

% -> various dimension -> constant k
When using $k$-WTA activation, we need to choose $k$. Yet it makes no sense to fix
$k$ throughout all layers of the neural network,
because these layers often have different output dimensions;
a small $k$ to one layer's dimension can be relatively large to the other.
Instead of specifying $k$, we introduce a parameter $\gamma\in(0,1)$ called \emph{sparsity ratio}.
If a layer has an output dimension $N$, then its $k$-WTA activation 
has $k=\floor{\gamma\cdot N}$.
Even though the sparsity ratio can be set differently for different layers, 
in practice we found no clear gain from introducing such a variation.
Therefore, we use a fixed $\gamma$---the only additional hyperparameter
needed for the neural network.

In convolutional neural networks (CNN), the output of a layer is a 
$C\times H\times W$ tensor. $C$ denotes the number of output channels;
$H$ and $W$ indicate the feature resolution.
While there are multiple choices of applying $k$-WTA on the tensor---for example,
one can apply $k$-WTA individually to each channel---empirically we found that
the most effective (and conceptually the simplest) way is to treat the tensor 
as a long $C\cdot H\cdot W\times 1$ vector input to the $k$-WTA activation.
Using $k$-WTA in this way is also backed by our theoretical understanding (see \secref{theory}).
% \todo{computational cost}

The runtime cost of computing a \kwta activation is asymptotically $O(N)$,
because finding $k$ largest values in a list is asymptotically equivalent to
finding the $k$-th largest value, which has an $O(N)$
complexity~\citep{cormen2009introduction}. This cost is comparable to ReLU's $O(N)$
cost on a $N$-length vector. Thus, replacing ReLU with \kwta introduces no
significant overhead.

\paraspace
%\vspace{-0.25mm}
\paragraph{Remark: other WTA-type activations.}
% why not LWTA
Relevant to \kwta is the \textit{local Winner-Take-All} (LWTA)
activation~\citep{srivastava2013compete,srivastava2014understanding}, which
divides each layer's output values into local groups 
%% like how max-pooling~\cite{riesenhuber1999hierarchical} splits layer output,
and applies WTA to each group individually. %% (i.e., $k$-WTA with $k=1$).
LWTA is similar to max-pooling~\citep{riesenhuber1999hierarchical} for dividing the layer output and choosing group maximums.
But unlike ReLU and max-pooling being $C^0$ continuous, LWTA 
and our $k$-WTA are both discontinuous with respect to the input.
The differences among ReLU, max-pooling, LWTA, and \kwta are illusrated in
Figure~\ref{fig:act}.

% -> long term memory not defend against adversarial threat
LWTA is motivated toward preventing catastrophic forgetting~\citep{mccloskey1989catastrophic},
whereas our use of $k$-WTA is for defending against adversarial threat. 
Both are discontinuous. But it remains unclear 
%%(in~\cite{srivastava2013compete,srivastava2014understanding})
what the LWTA's discontinuity properties are and how its discontinuities affect the network training. 
Our theoretical analysis (\secref{theory}), in contrast, 
sheds some light on these fundamental questions
about $k$-WTA, rationalizing its ability for improving adversarial robustness.
Indeed, our experiments confirm that $k$-WTA outperforms LWTA
in terms of robustness (see \appref{add}).

% -> no theoretical analysis -> ours is better; conceptually simpler

% ---> remain elusive in deep CNN
WTA-type activation, albeit originated decades ago,
remains elusive in modern neural networks.
Perhaps this is because it has not demonstrated a considerable improvement to
the network's standard test accuracy, though it 
can offer an accuracy comparable to ReLU~\citep{srivastava2013compete}.
Our analysis and proposed use of $k$-WTA and its enabled improvement 
in the context of adversarial defense
may suggest a renaissance of studying WTA.

%  Although $k$-WTA was popular in the early age of neural networks, it was not
%  been recognized in the recent revolution of deep learning.  However, we found
%  that $k$-WTA can replace ReLU in most modern deep architecture and achieves
%  similar test performance. By choosing appropriate value of $k$, the network can
%  have extra robustness against gradient-based attack.

\vspace{-0.05in}
\subsection{Training $k$-WTA Networks}\label{sec:train}
%\vspace{-0.5mm}
\vspace{-0.04in}
% -> in general no extra effort
% -> gamma is small has a challenge -> but is needed
\kwta networks require no special treatment in training.  Any optimization
algorithm (such as stochastic gradient descent) for training ReLU networks can be
directly used to train \kwta networks.%% as well. 

Our experiments have found that 
% when the sparsity ratio $\gamma$ is relatively large (larger than 0.1),
% training \kwta networks converges as fast as ReLU networks.
when the sparsity ratio $\gamma$ is relatively small ($\le 0.2$),
the network training converges slowly.
%% and may demand an arduous process.
This is not a surprise. A smaller $\gamma$ activates fewer neurons, effectively reducing more
of the layer width and in turn the network size, 
and the stripped ``subnetwork'' is much less expressive~\citep{srivastava2013compete}.
Since different training examples activate different subnetworks, collectively they make the training harder.

%    \begin{wrapfigure}[10]{r}{0.56\textwidth}
%    \vspace{-8.5mm}
%    \begin{minipage}{0.56\textwidth}
%    \begin{algorithm}[H]
%    \KwInput{starting sparsity ratio $\gamma_0$, step size $\delta$.}
%    
%    \KwOutput{A $k$-WTA network $f$ with a sparsity ratio $\gamma_1$}
%    
%    $\gamma = \gamma_0$; 
%    
%    Train $f$ for $t$ epochs; 
%    
%    \While{$\gamma>\gamma_1$}{
%    $\gamma = \gamma - \delta$;
%    
%    Train $f$ for one epoch;
%    }
%     \caption{Incremental training of a $k$-WTA network}\label{alg:train}
%    \end{algorithm}
%    \end{minipage}
%    \end{wrapfigure}
%  In general, training $k$-WTA networks does not require extra effort, any
%  optimization algorithms used in ReLU networks such as stochastic gradient
%  descent can be directly applied on $k$-WTA networks.  However, we found that
%  when $\gamma$ is small (usually less than 0.1), the $k$-WTA networks may
%  converge slowly and have relatively worse performance in terms of test accuracy
%  compare to networks have larger $\gamma$.  Intuitively, this is because a
%  smaller value $\gamma$ represent a smaller subnetwork for each evaluation,
%  which has weaker expressivity~\cite{srivastava2013compete}.  
Nevertheless, we prefer a smaller $\gamma$.
As we will discuss in the next section, a smaller $\gamma$
usually leads to better robustness against finding adversarial examples.
Therefore, to ease the training (when $\gamma$ is small), we propose to use 
an iterative fine-tuning approach.
Suppose the target sparsity ratio is $\gamma_1$. 
We first train the network with a larger sparsity ratio $\gamma_0$ using the standard 
training process. Then, we iteratively fine tune the network. In each iteration, 
we reduce its sparsity ratio by a small $\delta$ and train the network for two epochs.
The iteration repeats until
$\gamma_0$ is reduced to $\gamma_1$. %% (see \algref{train} for an outline).

This incremental process introduces little training overhead, because the cost of each fine tuning is 
negligible in comparison to training from scratch toward $\gamma_0$. We also note that this process is optional.
In practice we use it only when $\gamma<0.2$.
We show more experiments on the efficacy of the incremental training in \appref{inc}.

\vspace{-0.05in}
\section{Understanding $k$-WTA Discontinuity}\label{sec:theory}
%\vspace{-0.25mm}
\vspace{-0.05in}
% ---> discontinuity: 
%   - kWTA discontinuous; poorly understood
%   WTA-type functions in general, and \kwta in particular, are $C^0$ discontinuous. %% with respect to its input.  
%   But their discontinuity behaviors remain poorly understood. 
We now present our theoretical understanding of \kwta's discontinuity behavior
in the context of deep neural networks,
revealing some implication toward the network's adversarial robustness.

% - why discont. -> difference from ReLU
\paraspace
\paragraph{Activation pattern.}
To understand \kwta's discontinuity, consider one layer outputting values $\bm{x}$, passed through 
a \kwta activation, and followed by the next layer whose linear weight matrix is 
\begin{wrapfigure}[3]{r}{0.215\textwidth}
  \centering
  \vspace{-0mm}
  \includegraphics[width=0.215\textwidth]{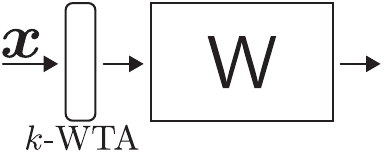}
\end{wrapfigure}
$\mathsf{W}$
(see adjacent figure). Then, the value fed into the next activation can be expressed as
$ \mathsf{W}\phi_k(\bm{x})$, where $\phi_k(\cdot)$ is the \kwta function defined in~\eq{kwta}.
Suppose the vector $\bm{x}$ has a length $l$.
We define the \kwta's \emph{activation pattern} under the input $\bm{x}$ as
\begin{equation}
\mathcal{A}(\bm{x}) \coloneqq \left\{i\in[l]\mid \text{$x_i$ is one of the $k$ largest values in $\bm{x}$}\right\} \subseteq[l].
\end{equation}
Here (and throughout this paper), we use $[l]$ to denote the integer set $\{1,2,...,l\}$.

\paraspace
\paragraph{Discontinuity.}
The activation pattern $\mathcal{A}(\bm{x})$ is a key notion for analyzing \kwta's discontinuity behavior.
Even an infinitesimal perturbation of $\bm{x}$ may change $\mathcal{A}(\bm{x})$: 
some element $i$ is removed from $\mathcal{A}(\bm{x})$ while another element $j$ is added in.
Corresponding to this change, in the evaluation of $\mathsf{W}\phi_k(\bm{x})$,
the contribution of $\sfW$'s column vector $\sfW_i$ vanishes while another column $\sfW_j$
suddenly takes effect.
% To evaluate $\mathsf{W}\phi_k(\bm{x})$, only those column vectors (of $\mathsf{W}$) whose indices
% are in $\mathcal{A}(\bm{x})$ take effect; others are ignored because the corresponding $\bm{x}$
% elements are set zero by \kwta.
% As a result, when $\mathcal{A}(\bm{x})$ changes, some column vectors of $\mathsf{W}$ effectively 
% vanish and others arise. 
%of the effective vectors in $\mathsf{W}$, 
%caused by an infinitesimal change of $\bm{x}$, 
It is this abrupt change that renders the result of $\mathsf{W}\phi_k(\bm{x})$ $C^0$ discontinuous.

Such a discontinuity jump can be arbitrarily large, 
because the column vectors $\mathsf{W}_i$ and $\sfW_j$ can be of any difference.
Once $\sfW$ is determined, the discontinuity jump then depends on 
the value of $x_i$ and $x_j$. 
As explained in \appref{jmp}, 
when the discontinuity occurs, $x_i$ and $x_j$ have about the same value, depending on
the choice of the sparsity ratio $\gamma$ (recall \secref{kwta_sub})---the smaller the $\gamma$ is, the larger the jump
will be. This relationship suggests that a smaller $\gamma$ will make the search of adversarial examples harder.
Indeed, this is confirmed through our experiments (see \appref{surf}).

%  These discontinuity behaviors differ fundamentally from the current mainstream activation functions such as ReLU and
%  hard tanh, which are all $C^0$ continuous.
%  It turns out that \kwta's $C^0$ discontinuity is the 
%  the cornerstone for its ability of improving the neural network's robustness.

% -> decades ago, poorly analyzed 
\paraspace
\paragraph{Piecewise linearity.}
% \begin{wrapfigure}[6]{r}{0.25\textwidth}
%   \centering
%   \vspace{-4mm}
%   \includegraphics[width=0.3\textwidth]{figs/placeholder2.jpg}
%   \vspace{-5.5mm}
%   \caption{ \textbf{Motivating example.} 
%   hello world
% \label{fig:illus}}
% \end{wrapfigure}
Now, consider an $n$-layer \kwta network, which can be expressed as
%\begin{equation}\label{eq:nlayer}
$
f(\bm{x}) = \sfW^{(1)} \cdot \phi_k(\sfW^{(2)} \cdot \phi_k ( \cdots  \phi_k(\sfW^{(n)} \bm{x} + \bb^{(n)}) ) + \bb^{(2)}) + \bb^{(1)}
$,
%\end{equation}
where $\sfW^{(i)}$ and $\bb^{(i)}$ are the $i$-th layer's weight matrix and bias vector, respectively.
If the activation patterns of all layers are fixed, then $f(\bm{x})$ is a linear function.
When the activation pattern changes, $f(\bm{x})$ switches from one linear function to another linear function.
Over the entire space of $\bm{x}$, $f(\bm{x})$ is \emph{piecewise} linear.
The specific activation patterns of all layers define a specific linear piece of the function, 
or a \emph{linear region} (following the notion introduced by~\cite{montufar2014number}).
Conventional ReLU (or hard tanh) networks also represent piecewise linear functions and their linear regions 
are joined together at their boundaries, whereas in \kwta networks the linear
regions are disconnected (see \figref{illus}).

\begin{figure*}[t]
    \centering
    %\vspace{-1mm}
    \includegraphics[width=0.98\textwidth]{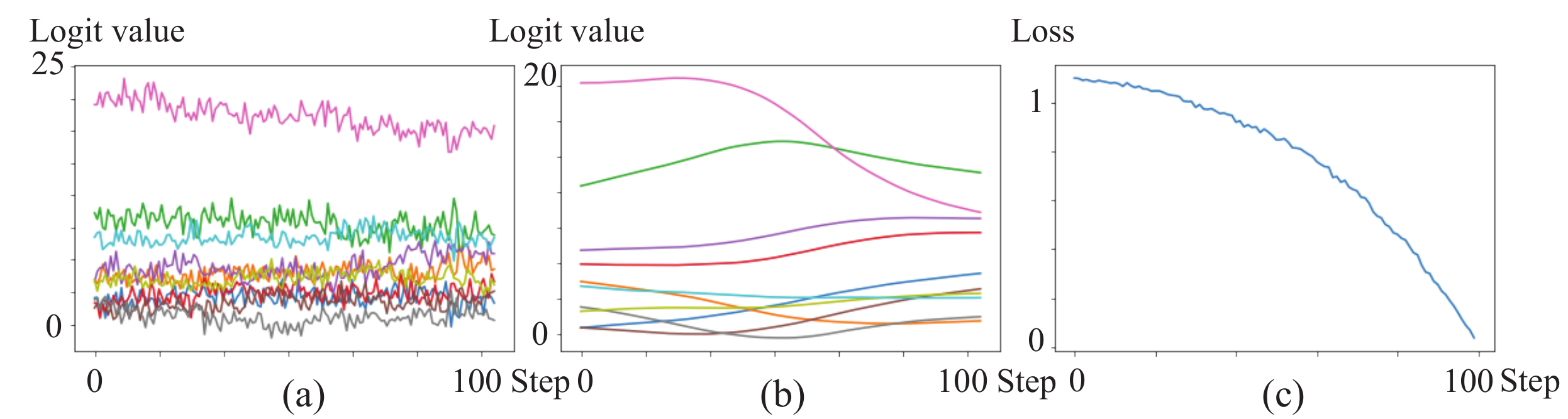}
    \vspace{-3mm}
    \caption{ %%Discontinuity for input and continuity for weight.
    \textbf{(a, b)} We plot the change of 10 logits values 
    when conducting untargeted PGD attack with 100 iterations. 
    X-axis indicates the perturbation size $\epsilon$ and Y-axis indicates the
    10 color-coded logits values. 
    %  Different color indicate the logits value correspond to
    %  different class.  
    \textbf{(a)} When we apply PGD attack on $k$-WTA ResNet18, the
    strong discontinuities w.r.t.~to input invalidate gradient estimation,
    effectively defending well against the attack. 
    \textbf{(b)} In contrast, for a ReLU ResNet18, PGD attack can
    easily find adversarial examples due to the model's smooth change w.r.t. input.
    \textbf{(c)} In the process of training \kwta ResNet18,
    the loss change w.r.t. model weights is largely smooth.
    Thus, the training is not harmed by \kwta's discontinuities.
    }
    \label{fig:dis}
    \vspace{-3mm}
\end{figure*}

\paraspace
\paragraph{Linear region density.}
Next, we gain some insight on the distribution of those linear regions.
This is of our interest because
if the linear regions are densely distributed, a small 
$\Delta\bm{x}$ perturbation from any data point $\bm{x}$ will likely cross the boundary 
of the linear region where $\bm{x}$ locates.
Whenever boundary crossing occurs, the gradient becomes undefined
(see \figref{dis}-a).
%% and experience the discontinuity in the output.

\begin{wrapfigure}[4]{r}{0.215\textwidth}
  \centering
  \vspace{-4mm}
  \includegraphics[width=0.215\textwidth]{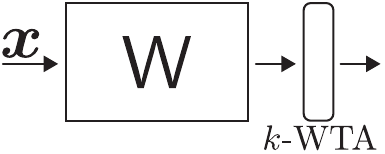}
\end{wrapfigure}
For the purpose of analysis, consider an input $\bm{x}$ passing through 
a layer followed by a \kwta activation (see adjacent figure). The output 
from the activation is $\phi_k(\sfW\bm{x}+\bb)$.
We would like to understand, when $\bm{x}$ is  %perturbed by $\Delta\bm{x}$
changed into $\bm{x}'$, how likely the activation pattern of $\phi_k$ will change.
First, notice that if $\bm{x}'$ and $\bm{x}$ satisfy $\bm{x}'=c\cdot\bm{x}$ with some $c>0$,
the activation pattern remains unchanged.
Therefore, we introduce a notation $\sfd(\bm{x},\bm{x}')$
that measures the ``perpendicular'' distance between $\bm{x}$ and $\bm{x}'$, 
one that satisfies $\bm{x}' = c\cdot\left(\bm{x} +\sfd(\bm{x},\bm{x}')\bm{x}_{\perp}\right)$
for some scalar $c$, where $\bm{x}_{\perp}$ is a unit vector perpendicular
to $\bm{x}$ and on the plane spanned by $\bm{x}$ and $\bm{x}'$.
With this notion, and if the elements in $\sfW$ is initialized by sampling
from $\mathcal{N}(0,\frac{1}{l})$ and $\bb$ is initialized as zero,
we find the following property:
%  Given a labeled data $(\bm{x},y)$, the gradient-based attack iteratively finds a small 
%  perturbation $\Delta\bm{x}$ such that the predicted label $f(\bm{x}+\bm{\Delta})$
%  deviates from $y$ as much as possible (see \appref{related}).
\begin{theorem}[Dense discontinuities]\label{thm:dense}
    Given any input $\bm{x}\in\mathbb{R}^m$ and some $\beta$, and 
    $\forall \bm{x}'\in\mathbb{R}^m$ such that $\frac{\sfd^2(\bm{x},\bm{x}')}{\|\bm{x}\|_2^2}\ge\beta$,
    if the following condition
    \begin{align*}
    l\geq \Omega\left( \left(\frac{m}{\gamma} \cdot \frac{1}{\beta}\right)\cdot \log\left(\frac{m}{\gamma} \cdot \frac{1}{\beta}\right) \right)
    \end{align*}
    is satisfied, then 
    with a probability at least $1-\cdot 2^{-m}$,
    we have $\mathcal{A}(\sfW\bm{x} + \bb)\not=\mathcal{A}(\sfW\bm{x}' + \bb)$. 
    %%where $x'=c\cdot (x + \Delta x)$ for some scaler $c$, and $\Delta x$ is perpendicular to $x$.
\end{theorem}
Here $l$ is the width of the layer, and $\gamma$ is again
the sparsity ratio in \kwta. 
This theorem informs us that 
the larger the layer width $l$ is, the smaller $\beta$---and thus the smaller 
perpendicular perturbation distance $\sfd(\bm{x},\bm{x}')$---is needed
to trigger a change of the activation pattern,
that is, as the layer width increases, the piecewise linear regions become finer
(see \appref{proofs} for proof and more discussion).
This property also echos a similar trend in ReLU networks, as pointed out by~\cite{raghu2017expressive}.

\paraspace
\paragraph{Why is the \kwta network trainable?}
While \kwta networks are highly discontinuous as revealed by \thmref{dense}
and our experiments (\figref{dis}-a), in practice we experience no difficulty on training these networks. 
Our next theorem sheds some light on 
the reason behind the training success.
\vspace{-2.5mm}
\begin{theorem}\label{thm:the_second_main_thm}
    Consider $N$ data points $\bm{x}_1,\bm{x}_2,\cdots,\bm{x}_N\in\mathbb{R}^m$.
    Suppose $\forall i\not= j, \frac{\bm{x}_i}{\|\bm{x}_i\|_2}\not =  \frac{\bm{x}_j}{\|\bm{x}_j\|_2}$.
 If $l$ is sufficiently large, then with a high probability, we have
 $\forall i\not =j,\mathcal{A}(\sfW\bm{x}_i + \bb)\cap \mathcal{A}(\sfW\bm{x}_j + \bb)=\varnothing$.
\end{theorem}
\vspace{-1mm}
This theorem is more formally stated in \thmref{disjoint} in \appref{proofs}
together with a proof there.
Intuitively speaking, it states that if the network is sufficiently wide,
then for any $i\ne j$, activation pattern of input data $\bm{x}_i$ 
is almost separated from that of $\bm{x}_j$.
Thus, the weights for predicting $\bm{x}_i$'s and $\bm{x}_j$'s labels can be optimized almost independently,
without changing their individual activation patterns.
In practice, the activation patterns of $\bm{x}_i$ and $\bm{x}_j$ are not fully separated
but weakly correlated. During the optimization, the activation pattern of a
data point $\bm{x}_i$ may change, but the chance is relatively low---a similar 
behavior has also been found in ReLU networks~\citep{ll18,du2018gradient,als18,als19,sy19}.

Further, notice that the training loss takes a summation over all training data points.
This means a weight update would change only a small set of activation patterns (since the chance
of having the pattern changed is low);
the discontinuous change on the loss value, 
after taking the summation, will be negligible (see \figref{dis}-c).
Thus, the discontinuities in \kwta is not harmful to network training.
%% so the next work can be trained successfully.

% piecewise linear

% ---> activation pattern

%   Although $k$-WTA has been proposed nearly 30 years, 
%   one important fact that has been omitted by previous research on $k$-WTA is that $k$-WTA is a 
%   discontinuous function with respect to its input.
%   This is counterintuitive because the activation functions used in modern deep networks are all continuous.
%   Although some activation functions are not continuous differentiable (e.g. ReLU, HardTanh),
%   they are still continuous.
%   Since modern deep networks are often trained by first-order gradient methods, 
%   it is natural to question that how a network with discontinuous activation still able to be trained.
% In this section, we will show that 

\vspace{-0.05in}
\section{Experimental Results}\label{sec:exp}
\vspace{-0.2mm}
\vspace{-0.05in}
We evaluate the robustness of \kwta networks under adversarial attacks. 
Our evaluation considers multiple training methods on different network architectures
(see details below).
When reporting statistics, we use 
$A_{rob}$ to indicate the robust accuracy under adversarial attacks
applied to the test dataset,
and $A_{std}$ to indicate the accuracy on the clean test data.
We use \kwta-$\gamma$ to represent \kwta activation with sparsity ratio $\gamma$.

%  When reporting statistics, we use 
%  $A_{rob}$ to indicate the model accuracy under adversarial attacks
%  applied to the test dataset,
%  and $A_{std}$ to indicate the accuracy on the clean test data.
%  We use \kwta-$\gamma$ to represent \kwta activation with sparsity ratio $\gamma$.

\vspace{-0.05in}
\subsection{Robustness under White-box Attacks}
\vspace{-0.05in}
The rationale behind \kwta activation is to destroy network gradients---information needed in white-box attacks. We therefore evaluate \kwta networks 
under multiple recently proposed white-box attack methods, including
%\textit{Fast Gradient Sign Method} (FGSM)~\citep{goodfellow2014explaining},
\textit{Projected Gradient Descent} (PGD)~\citep{madry2017towards},
\textit{Deepfool}~\citep{moosavi2016deepfool},
C\&W attack~\citep{carlini2017towards}, and
\textit{Momentum Iterative Method} (MIM)~\citep{dong2018boosting}.
% training methods
Since \kwta activation can be used in almost any training method, be it adversarial training or not,
we also consider multiple training methods, 
including natural (non-adversarial) training, adversarial training (AT)~\citep{madry2017towards}, 
TRADES~\citep{zhang2019theoretically} and free adversarial training (FAT)~\citep{shafahi2019adversarial}.

\begin{table}[t]
\vspace{-1mm}
\caption{Adversarial robustness on CIFAR-10 and SVHN datasets. 
$A_{rob}$ in the last column denotes the empirical worst-case robustness 
among different attacks (columns) for each network optimized by different training methods (row). 
The \textbf{bold} numbers indicate the best $A_{rob}$ robustness achieved on 
ReLU and \kwta networks by each training method.
    \label{tab:cifar_svhn_main}}
\centering

\begin{tabular}{ll|ccccccc}

\multicolumn{9}{c}{\textbf{CIFAR-10}} \\ \bottomrule
  Training& Activation & $A_{std}$  & PGD & C\&W & Deepfool & MIM & BB & $A_{rob}$ \\ \hline
\multirow{3}{*}{Natural} & ReLU &   92.9\%  & 0.0\% & 0.0\% & 1.5\% & 0.0\% & 18.9\% & 0.0\% \\
&\kwta-0.1   & 89.3\%  & 13.3\% & 27.9\% & 55.6\%& 13.1\% & 62.6\% &\textbf{13.1\%} \\ 
&\kwta-0.2   & 91.7\%  &  4.2\%&  6.2\% & 47.8\% & 3.9\% & 66.8\% &4.2\% \\ \hline
\multirow{3}{*}{AT} & ReLU  & 83.5\%  & 46.3\% & 43.6\% & 46.8\% & 45.9\% & 71.0\% & 43.6\% \\
&\kwta-0.1 & 78.9\%  & 51.4\% & 64.4\% & 70.4\% & 50.7\% &73.4\% & \textbf{50.7\%}  \\ 
&\kwta-0.2  & 81.4\%  & 48.4\% & 52.7\% & 66.1\% &  47.4\% & 73.5\% & 47.4\% \\ \hline
\multirow{3}{*}{TRADES} & ReLU & 79.7\% & 49.8\% & 52.3\% & 57.6\% & 49.9\%  &  70.6\%& 49.8\% \\
&\kwta-0.1  & 76.6\% & 55.0\% & 62.2\% & 66.0\% & 57.5\% & 72.3\% & \textbf{55.0\%}  \\ 
&\kwta-0.2  & 80.4\% & 51.5\% & 57.7\% & 63.9\% & 53.4\% & 74.7\% & 51.5\%  \\ \hline
\multirow{3}{*}{FAT} & ReLU  & 82.6\% & 42.7\% & 44.4\% & 49.7\% & 41.6\% & 73.4\% & 41.6\% \\
&\kwta-0.1  & 78.4\%  & 51.7\% & 66.3\% & 72.4\% & 49.1\% & 72.3\% & \textbf{49.1\%} \\
&\kwta-0.2  & 82.8\% & 48.4\% & 60.5\% & 67.2\% & 46.7\% &  76.8\% &  46.7\% \\ \hline
\toprule
\end{tabular}

\begin{tabular}{ll|ccccccc}

\multicolumn{9}{c}{\textbf{SVHN}} \\ \bottomrule
  Training& Activation & $A_{std}$  & PGD & C\&W & Deepfool & MIM & BB & $A_{rob}$ \\ \hline
\multirow{3}{*}{Natural} & ReLU &   95.1\%  & 0.0\% & 0.0\% & 2.5\% & 0.0\% & 14.7\% & 0.0\%  \\
&\kwta-0.1   & 92.6\%  &  10.2\% & 19.5\% & 88.7\%& 11.6\% & 51.4\% & \textbf{10.2\%} \\ 
&\kwta-0.2   & 93.8\%  &   4.3\%&  8.0\%&   86.8\%  & 8.3\% & 56.7\% & 4.3\% \\ \hline
\multirow{3}{*}{AT} & ReLU  & 84.2\%  & 44.5\% & 42.7\% & 70.3\% & 48.4\%  & 77.7\%& 42.7\% \\
&\kwta-0.1 & 79.9\%  & 62.2\%   & 65.7\% & 71.5\%  & 56.9\% & 76.1\%   &\textbf{56.9\%}\\ 
&\kwta-0.2  & 82.4\%  &  53.2\% & 63.6\% & 77.4\% & 52.3\% & 74.2\% & 52.3\%  \\ \hline
\multirow{3}{*}{TRADES} & ReLU & 84.7\% &  47.4\% & 51.6\% & 76.9\% &49.6\% & 76.5\% & 47.4\% \\
&\kwta-0.1  & 81.6\% & 61.3\% & 77.4\% & 79.4\% & 58.3\% &  78.1\% & \textbf{58.3\%} \\ 
&\kwta-0.2  & 85.4\% & 56.7\% & 59.2\% & 71.6\% & 54.5\% &  79.3\% & 54.5\% \\ \hline
\multirow{3}{*}{FAT} & ReLU  & 85.9\% & 40.8\% & 46.2\% & 76.1\% & 39.9\% & 76.9\% & 40.8\% \\
&\kwta-0.1  & 85.5\%  & 57.7\% & 70.0\% & 77.0\% & 62.8\% & 75.6\% & \textbf{57.7\%}  \\
&\kwta-0.2  & 86.8\% & 54.3\% & 64.3\% & 74.7\% & 55.2\% &  74.4\% &54.3\% \\ \hline
\toprule
\end{tabular}
% \vspace{0.3mm}
\vspace{-5mm}
\end{table}

In addition, we evaluate the robustness under transfer-based Black-box (BB) attacks~\citep{papernot2017practical}.
The black-box threat model requires no knowledge about network architecture and parameters.
Thus, we use a pre-trained VGG19 network~\citep{simonyan2014very} as the source
model to generate adversarial examples using PGD. As demonstrated by \citet{su2018robustness}, 
VGG networks have the strongest transferability among different architectures. 

In each setup, we compare the robust accuracy of \kwta networks with standard ReLU networks
on three datasets, CIFAR-10, SVHN, and MNIST. Results on the former two are reported
in \tabref{cifar_svhn_main}, while the latter is reported in \appref{mnist}.
We use ResNet-18 for CIFAR-10 and SVHN. The perturbation range
is 0.031 (CIFAR-10) and 0.047 (SVHN) for pixels ranging in $[0,1]$. 
More detailed training and attacking settings are reported in \appref{settings}.  

% When testing on MNIST dataset, we use a 4-layer
% convolutional network trained by SGD with a learning rate 0.01 for 20 epochs.
% The CNN architecture details are provided in \cmt{Appendix X},  
% and the \kwta CNNs in our test are all trained from scratch without incremental fine-tuning.
% We use $\gamma=0.1$ for all \kwta activation in our MNIST experiment.

% We further evaluate our results on CIFAR-10 and SVHN dataset. 
%  To evaluate the effectiveness of \kwta networks,
%  we use \textit{Fast Gradient Sign Method} (FGSM)~\citep{goodfellow2014explaining},
%  \textit{Projected Gradient Descent} (PGD)~\citep{madry2017towards},
%  \textit{Deepfool}~\citep{moosavi2016deepfool}, ~\citet{carlini2017towards} (C\&W) attack, 
%  \textit{Momentum Iterative Method} (MIM)~\citep{dong2018boosting} 

\textbf{The main takeaway} from these experiments (in \tabref{cifar_svhn_main}) 
is that \kwta is able to universally improve the white-box robustness, regardless of the training methods.
The \kwta robustness under black-box attacks is not always significantly better than ReLU networks.
But black-box attacks, due to the lack of network information, are generally much harder
than white-box attacks. In this sense, white-box attacks make the networks more vulnerable,
and \kwta is able to improve a network's worst-case robustness. This improvement is not tied to
any specific training method, achieved with
no significant overhead, just by a simple replacement of ReLU with \kwta.

% \paragraph{Results.}
% As shown in \tabref{cifar_svhn_main}, 

\citet{obfuscated-gradients} showed that gradient-based defenses 
may render the network more vulnerable under black-box attacks than 
under gradient-based white-box attacks.
% achieves better performance than gradient-based attack. 
However, we have not observed this behavior in \kwta networks. 
Even under the strongest black-box attack, i.e., by generating adversarial examples from
an independently trained copy of the target network, gradient-based attacks are still
stronger (with higher successful rate) than black-box attacks (see \appref{add}).  
%% This indicate that \kwta indeed raise the lower bound of model's robustness.

Additional experiments include: \textbf{1)} tests under transfer attacks across 
two independently trained \kwta networks
and across \kwta and ReLU networks,
\textbf{2)} evaluation of \kwta performance on different network architectures,
and \textbf{3)} comparison of \kwta performance with the LWTA (recall \secref{kwta_sub}) performance. See \appref{add} for details.

\vspace{-0.07in}
\subsection{Varying sparsity ratio $\gamma$ and model architecture.}
\vspace{-0.06in}
We further evaluate our method on various network architectures with different sparsity ratios $\gamma$.
\figref{gamma_change} shows the standard test accuracies and robust accuracies 
against PGD attacks while $\gamma$ decreases. %% in our fine-tuning process.  
To test on different network architectures,
we apply \kwta to ResNet18, DenseNet121 and Wide ResNet (WRN-22-12).
In each case, starting from $\gamma=0.2$, we decrease $\gamma$ using incremental fine-tuning.
We then evaluate the robust accuracy on CIFAR dataset, 
taking 20-iteration PGD attacks with a perturbation range $\epsilon=0.31$ for
pixels ranging in $[0, 1]$.

We find that when $\gamma$ is larger than $\sim 0.1$, reducing $\gamma$ 
has little effect on the standard accuracy, but increases the robust accuracy. 
When $\gamma$ is smaller than $\sim 0.1$, reducing $\gamma$ drastically lowers both the standard
and robust accuracies.
The peaks in the $A_{rob}$ curves (\figref{gamma_change}-right) are consistent with our theoretical
understanding: 
\thmref{dense} suggests that when $l$ is fixed, a smaller $\gamma$ 
tends to sparsify the linear region boundaries, 
exposing more gradients to the attacker.
Meanwhile, as also discussed in \secref{theory}, a smaller $\gamma$ leads to a
larger discontinuity jump and thus tends to improve the robustness.

%  But if $\gamma$ is too small, the activated subnetwork for each data 
%  becomes too limited and thus the performance starts to drop.
%  makes the linear regions denser and thus the network becomes more robust.

\begin{figure*}[t]
	\centering
	\vspace{-1mm}
	\includegraphics[width=0.88\textwidth]{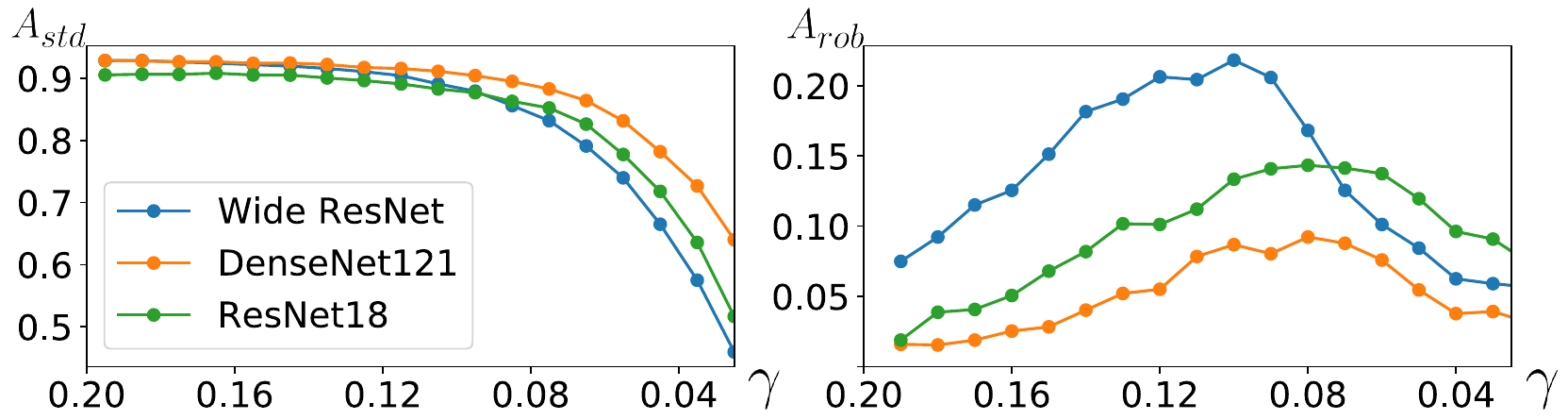}
	\vspace{-3mm}
	\caption{Robustness changing w.r.t. $\gamma$ on CIFAR. 
    When $\gamma$ decreases, the standard test accuracy (left) starts to drop after a certain point.
    The robust accuracy (right) first increases then decreases.
	}
    \label{fig:gamma_change}
	\vspace{-2mm}
\end{figure*}

\vspace{-0.05in}
\subsection{Loss Landscape in Gradient-Based Attacks}\label{sec:loss_land}
\vspace{-0.05in}
\begin{figure*}[t]
	\centering
	%\vspace{-1mm}
	\includegraphics[width=0.96\textwidth]{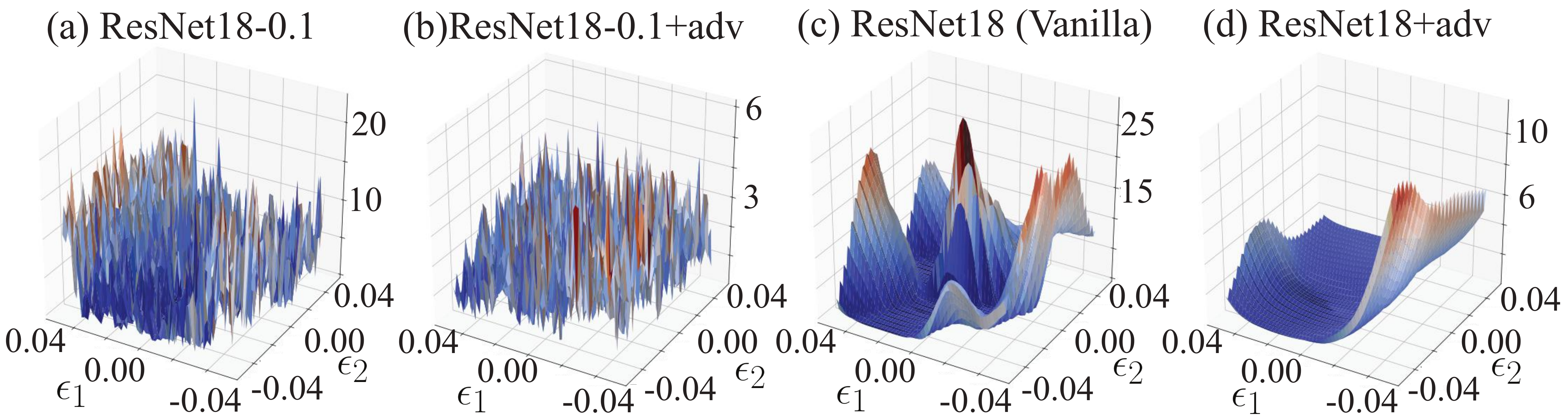}
	\vspace{-2mm}
	\caption{Gradient-based attack's loss landscapes in $k$-WTA \textbf{(a, b)} 
        and conventional ReLU models \textbf{(c, d)}.
            {(a,b)} \kwta Models have much more non-convex and non-smooth 
            landscapes.
	Also, the model optimized by adversarial training (b) has a lower absolute value of loss.
	}
    \label{fig:loss}
    \vspace{-4mm}
\end{figure*}
We now empirically unveil why \kwta is able to improve the network's robustness 
(in addition to our theoretical analysis in \secref{theory}).
Here we visualize the attacker's loss landscape in gradient-based attacks
in order to reveal the landscape change caused by \kwta.
Similar to the analysis in~\citet{tramer2017ensemble}, 
we plot the attack loss of a model with respect to its input
on points $\bm{x}^\prime = \bm{x} + \epsilon_1 \bm{g}_1 + \epsilon_2 \bm{g}_2$,
where $\bm{x}$ is a test sample from CIFAR test set, 
$\bm{g}_1$ is the direction of the loss gradient with respect to the input,
$\bm{g}_2$ is another random direction, $\epsilon_1$ and $\epsilon_2$ 
sweep in the range of $[-0.04, 0.04]$, with 50 samples each. 
This results in a 3D landscape plot with 2500 data points (\figref{loss}).

As shown~in \figref{loss}, \kwta models (with $\gamma=0.1$) have a highly
non-convex and non-smooth loss landscape. Thus, the estimated gradient 
is hardly useful for adversarial searches. 
This explains why $k$-WTA models can effectively resist gradient-based attacks.
In contrast, ReLU models have a much smoother loss surface, 
from which adversarial examples can be easily found using gradient descent.

Inspecting the range of loss values in \figref{loss},
%%The difference in absolute value of the loss shows that 
we find that    %%% the loss landscape of models optimized by 
adversarial training tends to 
compress the loss landscape's dynamic range
in both the gradient
direction and the other random direction, 
making the dynamic range smaller than that of the models without adversarial training.
%compare to a regular network, which is
This phenomenon has been observed in ReLU networks~\citep{madry2017towards,tramer2017ensemble}. 
Interestingly, $k$-WTA models manifest a similar behavior (\figref{loss}-a,b).
Moreover, we find that in \kwta models a larger $\gamma$ 
leads to a smoother loss surface than a smaller $\gamma$ (see 
\appref{surf} for more details).

\vspace{-0.05in}
\vspace{-0.25mm}
\section{Conclusion}
\vspace{-0.25mm}
\vspace{-0.05in}
This paper proposes to replace widely used activation functions with 
the \kwta activation
for improving the neural network's robustness against 
adversarial attacks. This is the only change we advocate. 
%especially gradient-based attacks
The underlying idea is to embrace the discontinuities introduced by \kwta functions to 
make the search for adversarial examples more challenging. 
Our method comes almost for free,
harmless to network training, and
readily useful in the current paradigm of neural networks.

\bibliographystyle{iclr2020_conference}
\bibliography{ref}

\begin{thebibliography}{62}
\providecommand{\natexlab}[1]{#1}
\providecommand{\url}[1]{\texttt{#1}}
\expandafter\ifx\csname urlstyle\endcsname\relax
  \providecommand{\doi}[1]{doi: #1}\else
  \providecommand{\doi}{doi: \begingroup \urlstyle{rm}\Url}\fi

\bibitem[Allen-Zhu et~al.(2019{\natexlab{a}})Allen-Zhu, Li, and Song]{als18}
Zeyuan Allen-Zhu, Yuanzhi Li, and Zhao Song.
\newblock On the convergence rate of training recurrent neural networks.
\newblock In \emph{NeurIPS}. \url{https://arxiv.org/pdf/1810.12065},
  2019{\natexlab{a}}.

\bibitem[Allen-Zhu et~al.(2019{\natexlab{b}})Allen-Zhu, Li, and Song]{als19}
Zeyuan Allen-Zhu, Yuanzhi Li, and Zhao Song.
\newblock A convergence theory for deep learning via over-parameterization.
\newblock In \emph{ICML}. \url{https://arxiv.org/pdf/1811.03962},
  2019{\natexlab{b}}.

\bibitem[Athalye et~al.(2017)Athalye, Engstrom, Ilyas, and
  Kwok]{athalye2017synthesizing}
Anish Athalye, Logan Engstrom, Andrew Ilyas, and Kevin Kwok.
\newblock Synthesizing robust adversarial examples.
\newblock \emph{arXiv preprint arXiv:1707.07397}, 2017.

\bibitem[Athalye et~al.(2018)Athalye, Carlini, and
  Wagner]{obfuscated-gradients}
Anish Athalye, Nicholas Carlini, and David Wagner.
\newblock Obfuscated gradients give a false sense of security: Circumventing
  defenses to adversarial examples.
\newblock In \emph{Proceedings of the 35th International Conference on Machine
  Learning, {ICML} 2018}, July 2018.
\newblock URL \url{https://arxiv.org/abs/1802.00420}.

\bibitem[Barreno et~al.(2006)Barreno, Nelson, Sears, Joseph, and
  Tygar]{barreno2006can}
Marco Barreno, Blaine Nelson, Russell Sears, Anthony~D Joseph, and J~Doug
  Tygar.
\newblock Can machine learning be secure?
\newblock In \emph{Proceedings of the 2006 ACM Symposium on Information,
  computer and communications security}, pp.\  16--25. ACM, 2006.

\bibitem[Barreno et~al.(2010)Barreno, Nelson, Joseph, and
  Tygar]{barreno2010security}
Marco Barreno, Blaine Nelson, Anthony~D Joseph, and J~Doug Tygar.
\newblock The security of machine learning.
\newblock \emph{Machine Learning}, 81\penalty0 (2):\penalty0 121--148, 2010.

\bibitem[Brendel et~al.(2017)Brendel, Rauber, and Bethge]{brendel2017decision}
Wieland Brendel, Jonas Rauber, and Matthias Bethge.
\newblock Decision-based adversarial attacks: Reliable attacks against
  black-box machine learning models.
\newblock \emph{arXiv preprint arXiv:1712.04248}, 2017.

\bibitem[Buckman et~al.(2018)Buckman, Roy, Raffel, and Goodfellow]{thermo2018}
Jacob Buckman, Aurko Roy, Colin Raffel, and Ian Goodfellow.
\newblock Thermometer encoding: One hot way to resist adversarial examples.
\newblock 2018.
\newblock URL \url{https://openreview.net/pdf?id=S18Su--CW}.

\bibitem[Carlini \& Wagner(2017)Carlini and Wagner]{carlini2017towards}
Nicholas Carlini and David Wagner.
\newblock Towards evaluating the robustness of neural networks.
\newblock In \emph{2017 IEEE Symposium on Security and Privacy (SP)}, pp.\
  39--57. IEEE, 2017.

\bibitem[Chen \& Jordan(2019)Chen and Jordan]{chen2019boundary}
Jianbo Chen and Michael~I Jordan.
\newblock Boundary attack++: Query-efficient decision-based adversarial attack.
\newblock \emph{arXiv preprint arXiv:1904.02144}, 2019.

\bibitem[Cohen et~al.(2019)Cohen, Rosenfeld, and Kolter]{cohen2019certified}
Jeremy~M Cohen, Elan Rosenfeld, and J~Zico Kolter.
\newblock Certified adversarial robustness via randomized smoothing.
\newblock \emph{arXiv preprint arXiv:1902.02918}, 2019.

\bibitem[Cormen et~al.(2009)Cormen, Leiserson, Rivest, and
  Stein]{cormen2009introduction}
Thomas~H Cormen, Charles~E Leiserson, Ronald~L Rivest, and Clifford Stein.
\newblock \emph{Introduction to algorithms}.
\newblock MIT press, 2009.

\bibitem[Deng et~al.(2009)Deng, Dong, Socher, Li, Li, and
  Fei-Fei]{deng2009imagenet}
Jia Deng, Wei Dong, Richard Socher, Li-Jia Li, Kai Li, and Li~Fei-Fei.
\newblock Imagenet: A large-scale hierarchical image database.
\newblock In \emph{2009 IEEE conference on computer vision and pattern
  recognition}, pp.\  248--255. Ieee, 2009.

\bibitem[Dhillon et~al.(2018)Dhillon, Azizzadenesheli, Bernstein, Kossaifi,
  Khanna, Lipton, and Anandkumar]{s.2018stochastic}
Guneet~S. Dhillon, Kamyar Azizzadenesheli, Jeremy~D. Bernstein, Jean Kossaifi,
  Aran Khanna, Zachary~C. Lipton, and Animashree Anandkumar.
\newblock Stochastic activation pruning for robust adversarial defense.
\newblock In \emph{International Conference on Learning Representations}, 2018.
\newblock URL \url{https://openreview.net/forum?id=H1uR4GZRZ}.

\bibitem[Dong et~al.(2018)Dong, Liao, Pang, Su, Zhu, Hu, and
  Li]{dong2018boosting}
Yinpeng Dong, Fangzhou Liao, Tianyu Pang, Hang Su, Jun Zhu, Xiaolin Hu, and
  Jianguo Li.
\newblock Boosting adversarial attacks with momentum.
\newblock In \emph{Proceedings of the IEEE Conference on Computer Vision and
  Pattern Recognition}, pp.\  9185--9193, 2018.

\bibitem[Du et~al.(2018)Du, Zhai, Poczos, and Singh]{du2018gradient}
Simon~S Du, Xiyu Zhai, Barnabas Poczos, and Aarti Singh.
\newblock Gradient descent provably optimizes over-parameterized neural
  networks.
\newblock \emph{arXiv preprint arXiv:1810.02054}, 2018.

\bibitem[Goodfellow et~al.(2014)Goodfellow, Shlens, and
  Szegedy]{goodfellow2014explaining}
Ian~J Goodfellow, Jonathon Shlens, and Christian Szegedy.
\newblock Explaining and harnessing adversarial examples.
\newblock \emph{arXiv preprint arXiv:1412.6572}, 2014.

\bibitem[Grossberg(1982)]{grossberg1982contour}
Stephen Grossberg.
\newblock Contour enhancement, short term memory, and constancies in
  reverberating neural networks.
\newblock In \emph{Studies of mind and brain}, pp.\  332--378. Springer, 1982.

\bibitem[Guo et~al.(2018)Guo, Rana, Cisse, and van~der
  Maaten]{guo2018countering}
Chuan Guo, Mayank Rana, Moustapha Cisse, and Laurens van~der Maaten.
\newblock Countering adversarial images using input transformations.
\newblock In \emph{International Conference on Learning Representations}, 2018.
\newblock URL \url{https://openreview.net/forum?id=SyJ7ClWCb}.

\bibitem[He et~al.(2016)He, Zhang, Ren, and Sun]{he2016deep}
Kaiming He, Xiangyu Zhang, Shaoqing Ren, and Jian Sun.
\newblock Deep residual learning for image recognition.
\newblock In \emph{Proceedings of the IEEE conference on computer vision and
  pattern recognition}, pp.\  770--778, 2016.

\bibitem[Huang et~al.(2017)Huang, Liu, Van Der~Maaten, and
  Weinberger]{huang2017densely}
Gao Huang, Zhuang Liu, Laurens Van Der~Maaten, and Kilian~Q Weinberger.
\newblock Densely connected convolutional networks.
\newblock In \emph{Proceedings of the IEEE conference on computer vision and
  pattern recognition}, pp.\  4700--4708, 2017.

\bibitem[Huang et~al.(2015)Huang, Xu, Schuurmans, and
  Szepesvári]{huang2015learning}
Ruitong Huang, Bing Xu, Dale Schuurmans, and Csaba Szepesvári.
\newblock Learning with a strong adversary.
\newblock \emph{http://arxiv.org/abs/1511.03034}, 11 2015.

\bibitem[Kurakin et~al.(2016)Kurakin, Goodfellow, and
  Bengio]{kurakin2016adversarial}
Alexey Kurakin, Ian Goodfellow, and Samy Bengio.
\newblock Adversarial machine learning at scale.
\newblock \emph{arXiv preprint arXiv:1611.01236}, 2016.

\bibitem[Li \& Liang(2018)Li and Liang]{ll18}
Yuanzhi Li and Yingyu Liang.
\newblock Learning overparameterized neural networks via stochastic gradient
  descent on structured data.
\newblock In \emph{Advances in Neural Information Processing Systems}, pp.\
  8157--8166, 2018.

\bibitem[Lin et~al.(2019)Lin, Gan, and Han]{lin2018defensive}
Ji~Lin, Chuang Gan, and Song Han.
\newblock Defensive quantization: When efficiency meets robustness.
\newblock In \emph{International Conference on Learning Representations}, 2019.
\newblock URL \url{https://openreview.net/forum?id=ryetZ20ctX}.

\bibitem[Ma et~al.(2018)Ma, Li, Wang, Erfani, Wijewickrema, Schoenebeck, Houle,
  Song, and Bailey]{ma2018characterizing}
Xingjun Ma, Bo~Li, Yisen Wang, Sarah~M. Erfani, Sudanthi Wijewickrema, Grant
  Schoenebeck, Michael~E. Houle, Dawn Song, and James Bailey.
\newblock Characterizing adversarial subspaces using local intrinsic
  dimensionality.
\newblock In \emph{International Conference on Learning Representations}, 2018.
\newblock URL \url{https://openreview.net/forum?id=B1gJ1L2aW}.

\bibitem[Maass(2000{\natexlab{a}})]{maass2000computational}
Wolfgang Maass.
\newblock On the computational power of winner-take-all.
\newblock \emph{Neural computation}, 12\penalty0 (11):\penalty0 2519--2535,
  2000{\natexlab{a}}.

\bibitem[Maass(2000{\natexlab{b}})]{maass2000neural}
Wolfgang Maass.
\newblock Neural computation with winner-take-all as the only nonlinear
  operation.
\newblock In \emph{Advances in neural information processing systems}, pp.\
  293--299, 2000{\natexlab{b}}.

\bibitem[Madry et~al.(2017)Madry, Makelov, Schmidt, Tsipras, and
  Vladu]{madry2017towards}
Aleksander Madry, Aleksandar Makelov, Ludwig Schmidt, Dimitris Tsipras, and
  Adrian Vladu.
\newblock Towards deep learning models resistant to adversarial attacks.
\newblock \emph{arXiv preprint arXiv:1706.06083}, 2017.

\bibitem[Majani et~al.(1989)Majani, Erlanson, and Abu-Mostafa]{majani1989k}
E~Majani, Ruth Erlanson, and Yaser~S Abu-Mostafa.
\newblock On the k-winners-take-all network.
\newblock In \emph{Advances in neural information processing systems}, pp.\
  634--642, 1989.

\bibitem[McCloskey \& Cohen(1989)McCloskey and
  Cohen]{mccloskey1989catastrophic}
Michael McCloskey and Neal~J Cohen.
\newblock Catastrophic interference in connectionist networks: The sequential
  learning problem.
\newblock In \emph{Psychology of learning and motivation}, volume~24, pp.\
  109--165. Elsevier, 1989.

\bibitem[Montufar et~al.(2014)Montufar, Pascanu, Cho, and
  Bengio]{montufar2014number}
Guido~F Montufar, Razvan Pascanu, Kyunghyun Cho, and Yoshua Bengio.
\newblock On the number of linear regions of deep neural networks.
\newblock In \emph{Advances in neural information processing systems}, pp.\
  2924--2932, 2014.

\bibitem[Moosavi-Dezfooli et~al.(2016)Moosavi-Dezfooli, Fawzi, and
  Frossard]{moosavi2016deepfool}
Seyed-Mohsen Moosavi-Dezfooli, Alhussein Fawzi, and Pascal Frossard.
\newblock Deepfool: a simple and accurate method to fool deep neural networks.
\newblock In \emph{Proceedings of the IEEE conference on computer vision and
  pattern recognition}, pp.\  2574--2582, 2016.

\bibitem[Narodytska \& Kasiviswanathan(2016)Narodytska and
  Kasiviswanathan]{narodytska2016simple}
Nina Narodytska and Shiva~Prasad Kasiviswanathan.
\newblock Simple black-box adversarial perturbations for deep networks.
\newblock \emph{arXiv preprint arXiv:1612.06299}, 2016.

\bibitem[Papernot et~al.(2017)Papernot, McDaniel, Goodfellow, Jha, Celik, and
  Swami]{papernot2017practical}
Nicolas Papernot, Patrick McDaniel, Ian Goodfellow, Somesh Jha, Z~Berkay Celik,
  and Ananthram Swami.
\newblock Practical black-box attacks against machine learning.
\newblock In \emph{Proceedings of the 2017 ACM on Asia conference on computer
  and communications security}, pp.\  506--519. ACM, 2017.

\bibitem[Raghu et~al.(2017)Raghu, Poole, Kleinberg, Ganguli, and
  Dickstein]{raghu2017expressive}
Maithra Raghu, Ben Poole, Jon Kleinberg, Surya Ganguli, and Jascha~Sohl
  Dickstein.
\newblock On the expressive power of deep neural networks.
\newblock In \emph{Proceedings of the 34th International Conference on Machine
  Learning-Volume 70}, pp.\  2847--2854. JMLR. org, 2017.

\bibitem[Rauber et~al.(2017)Rauber, Brendel, and Bethge]{rauber2017foolbox}
Jonas Rauber, Wieland Brendel, and Matthias Bethge.
\newblock Foolbox: A python toolbox to benchmark the robustness of machine
  learning models.
\newblock \emph{arXiv preprint arXiv:1707.04131}, 2017.
\newblock URL \url{http://arxiv.org/abs/1707.04131}.

\bibitem[Riesenhuber \& Poggio(1999)Riesenhuber and
  Poggio]{riesenhuber1999hierarchical}
Maximilian Riesenhuber and Tomaso Poggio.
\newblock Hierarchical models of object recognition in cortex.
\newblock \emph{Nature neuroscience}, 2\penalty0 (11):\penalty0 1019, 1999.

\bibitem[Ross \& Doshi{-}Velez(2017)Ross and Doshi{-}Velez]{inputgradient}
Andrew~Slavin Ross and Finale Doshi{-}Velez.
\newblock Improving the adversarial robustness and interpretability of deep
  neural networks by regularizing their input gradients.
\newblock \emph{CoRR}, abs/1711.09404, 2017.
\newblock URL \url{http://arxiv.org/abs/1711.09404}.

\bibitem[Samangouei et~al.(2018)Samangouei, Kabkab, and
  Chellappa]{samangouei2018defensegan}
Pouya Samangouei, Maya Kabkab, and Rama Chellappa.
\newblock Defense-{GAN}: Protecting classifiers against adversarial attacks
  using generative models.
\newblock In \emph{International Conference on Learning Representations}, 2018.
\newblock URL \url{https://openreview.net/forum?id=BkJ3ibb0-}.

\bibitem[Shafahi et~al.(2019{\natexlab{a}})Shafahi, Huang, Studer, Feizi, and
  Goldstein]{shafahi2018are}
Ali Shafahi, W.~Ronny Huang, Christoph Studer, Soheil Feizi, and Tom Goldstein.
\newblock Are adversarial examples inevitable?
\newblock In \emph{International Conference on Learning Representations},
  2019{\natexlab{a}}.
\newblock URL \url{https://openreview.net/forum?id=r1lWUoA9FQ}.

\bibitem[Shafahi et~al.(2019{\natexlab{b}})Shafahi, Najibi, Ghiasi, Xu,
  Dickerson, Studer, Davis, Taylor, and Goldstein]{shafahi2019adversarial}
Ali Shafahi, Mahyar Najibi, Amin Ghiasi, Zheng Xu, John Dickerson, Christoph
  Studer, Larry~S Davis, Gavin Taylor, and Tom Goldstein.
\newblock Adversarial training for free!
\newblock \emph{arXiv preprint arXiv:1904.12843}, 2019{\natexlab{b}}.

\bibitem[Sharif et~al.(2016)Sharif, Bhagavatula, Bauer, and
  Reiter]{sharif2016accessorize}
Mahmood Sharif, Sruti Bhagavatula, Lujo Bauer, and Michael~K. Reiter.
\newblock Accessorize to a crime: Real and stealthy attacks on state-of-the-art
  face recognition.
\newblock In \emph{Proceedings of the 2016 ACM SIGSAC Conference on Computer
  and Communications Security}, CCS '16, pp.\  1528--1540, New York, NY, USA,
  2016. ACM.
\newblock ISBN 978-1-4503-4139-4.
\newblock \doi{10.1145/2976749.2978392}.
\newblock URL \url{http://doi.acm.org/10.1145/2976749.2978392}.

\bibitem[Simonyan \& Zisserman(2014)Simonyan and Zisserman]{simonyan2014very}
Karen Simonyan and Andrew Zisserman.
\newblock Very deep convolutional networks for large-scale image recognition.
\newblock \emph{arXiv preprint arXiv:1409.1556}, 2014.

\bibitem[Song et~al.(2018)Song, Kim, Nowozin, Ermon, and
  Kushman]{song2018pixeldefend}
Yang Song, Taesup Kim, Sebastian Nowozin, Stefano Ermon, and Nate Kushman.
\newblock Pixeldefend: Leveraging generative models to understand and defend
  against adversarial examples.
\newblock In \emph{International Conference on Learning Representations}, 2018.
\newblock URL \url{https://openreview.net/forum?id=rJUYGxbCW}.

\bibitem[Song \& Yang(2019)Song and Yang]{sy19}
Zhao Song and Xin Yang.
\newblock Quadratic suffices for over-parametrization via matrix chernoff
  bound.
\newblock \emph{arXiv preprint arXiv:1906.03593}, 2019.

\bibitem[Srivastava et~al.(2013)Srivastava, Masci, Kazerounian, Gomez, and
  Schmidhuber]{srivastava2013compete}
Rupesh~K Srivastava, Jonathan Masci, Sohrob Kazerounian, Faustino Gomez, and
  J\"{u}rgen Schmidhuber.
\newblock Compete to compute.
\newblock In \emph{Advances in Neural Information Processing Systems 26}, pp.\
  2310--2318. 2013.
\newblock URL \url{http://papers.nips.cc/paper/5059-compete-to-compute.pdf}.

\bibitem[Srivastava et~al.(2014)Srivastava, Masci, Gomez, and
  Schmidhuber]{srivastava2014understanding}
Rupesh~Kumar Srivastava, Jonathan Masci, Faustino Gomez, and J{\"u}rgen
  Schmidhuber.
\newblock Understanding locally competitive networks.
\newblock \emph{arXiv preprint arXiv:1410.1165}, 2014.

\bibitem[Su et~al.(2018{\natexlab{a}})Su, Zhang, Chen, Yi, Chen, and
  Gao]{robustcost2018}
Dong Su, Huan Zhang, Hongge Chen, Jinfeng Yi, Pin-Yu Chen, and Yupeng Gao.
\newblock Is robustness the cost of accuracy? -- a comprehensive study on the
  robustness of 18 deep image classification models.
\newblock In Vittorio Ferrari, Martial Hebert, Cristian Sminchisescu, and Yair
  Weiss (eds.), \emph{Computer Vision -- ECCV 2018}, pp.\  644--661, Cham,
  2018{\natexlab{a}}. Springer International Publishing.

\bibitem[Su et~al.(2018{\natexlab{b}})Su, Zhang, Chen, Yi, Chen, and
  Gao]{su2018robustness}
Dong Su, Huan Zhang, Hongge Chen, Jinfeng Yi, Pin-Yu Chen, and Yupeng Gao.
\newblock Is robustness the cost of accuracy?--a comprehensive study on the
  robustness of 18 deep image classification models.
\newblock In \emph{Proceedings of the European Conference on Computer Vision
  (ECCV)}, pp.\  631--648, 2018{\natexlab{b}}.

\bibitem[Su et~al.(2019)Su, Vargas, and Sakurai]{su2019one}
Jiawei Su, Danilo~Vasconcellos Vargas, and Kouichi Sakurai.
\newblock One pixel attack for fooling deep neural networks.
\newblock \emph{IEEE Transactions on Evolutionary Computation}, 2019.

\bibitem[Szegedy et~al.(2014)Szegedy, Zaremba, Sutskever, Bruna, Erhan,
  Goodfellow, and Fergus]{intriguing2014}
Christian Szegedy, Wojciech Zaremba, Ilya Sutskever, Joan Bruna, Dumitru Erhan,
  Ian Goodfellow, and Rob Fergus.
\newblock Intriguing properties of neural networks.
\newblock In \emph{International Conference on Learning Representations}, 2014.
\newblock URL \url{http://arxiv.org/abs/1312.6199}.

\bibitem[Thys et~al.(2019)Thys, Van~Ranst, and Goedem{\'e}]{thys2019fooling}
Simen Thys, Wiebe Van~Ranst, and Toon Goedem{\'e}.
\newblock Fooling automated surveillance cameras: adversarial patches to attack
  person detection.
\newblock \emph{arXiv preprint arXiv:1904.08653}, 2019.

\bibitem[Tram{\`e}r et~al.(2017)Tram{\`e}r, Kurakin, Papernot, Goodfellow,
  Boneh, and McDaniel]{tramer2017ensemble}
Florian Tram{\`e}r, Alexey Kurakin, Nicolas Papernot, Ian Goodfellow, Dan
  Boneh, and Patrick McDaniel.
\newblock Ensemble adversarial training: Attacks and defenses.
\newblock \emph{arXiv preprint arXiv:1705.07204}, 2017.

\bibitem[Tsipras et~al.(2018)Tsipras, Santurkar, Engstrom, Turner, and
  Madry]{tsipras2018robustness}
Dimitris Tsipras, Shibani Santurkar, Logan Engstrom, Alexander Turner, and
  Aleksander Madry.
\newblock Robustness may be at odds with accuracy.
\newblock \emph{stat}, 1050:\penalty0 11, 2018.

\bibitem[Weng et~al.(2018)Weng, Zhang, Chen, Song, Hsieh, Boning, Dhillon, and
  Daniel]{weng2018towards}
Tsui-Wei Weng, Huan Zhang, Hongge Chen, Zhao Song, Cho-Jui Hsieh, Duane Boning,
  Inderjit~S Dhillon, and Luca Daniel.
\newblock Towards fast computation of certified robustness for relu networks.
\newblock 2018.

\bibitem[Wojtaszczyk(1996)]{wojtaszczyk1996banach}
Przemyslaw Wojtaszczyk.
\newblock \emph{Banach spaces for analysts}, volume~25.
\newblock Cambridge University Press, 1996.

\bibitem[Xie et~al.(2018{\natexlab{a}})Xie, Wang, Zhang, Ren, and
  Yuille]{xie2018mitigating}
Cihang Xie, Jianyu Wang, Zhishuai Zhang, Zhou Ren, and Alan Yuille.
\newblock Mitigating adversarial effects through randomization.
\newblock In \emph{International Conference on Learning Representations},
  2018{\natexlab{a}}.
\newblock URL \url{https://openreview.net/forum?id=Sk9yuql0Z}.

\bibitem[Xie et~al.(2018{\natexlab{b}})Xie, Wu, van~der Maaten, Yuille, and
  He]{xie2018feature}
Cihang Xie, Yuxin Wu, Laurens van~der Maaten, Alan Yuille, and Kaiming He.
\newblock Feature denoising for improving adversarial robustness.
\newblock \emph{arXiv preprint arXiv:1812.03411}, 2018{\natexlab{b}}.

\bibitem[Zagoruyko \& Komodakis(2016)Zagoruyko and
  Komodakis]{zagoruyko2016wide}
Sergey Zagoruyko and Nikos Komodakis.
\newblock Wide residual networks.
\newblock \emph{arXiv preprint arXiv:1605.07146}, 2016.

\bibitem[Zhang et~al.(2019)Zhang, Yu, Jiao, Xing, Ghaoui, and
  Jordan]{zhang2019theoretically}
Hongyang Zhang, Yaodong Yu, Jiantao Jiao, Eric~P Xing, Laurent~El Ghaoui, and
  Michael~I Jordan.
\newblock Theoretically principled trade-off between robustness and accuracy.
\newblock \emph{arXiv preprint arXiv:1901.08573}, 2019.

\bibitem[Zheng et~al.(2016)Zheng, Song, Leung, and Goodfellow]{Zheng_2016_CVPR}
Stephan Zheng, Yang Song, Thomas Leung, and Ian Goodfellow.
\newblock Improving the robustness of deep neural networks via stability
  training.
\newblock In \emph{The IEEE Conference on Computer Vision and Pattern
  Recognition (CVPR)}, June 2016.

\end{thebibliography}

%%%%%%%%%%%%%%%%%%%%%%%%%%%%%%%%%%%%%%%%%%%%%%%%%%%%%%%%%%%%%%%%%%%%%%%%%%%%%%%%%%%%%%%%
%% Appendix starts here
\newpage
\appendix
\begin{center}
\Large
\textbf{Supplementary Document}\\ 
\smallskip
\textbf{Enhancing Adversarial Defense by $k$-Winners-Take-All}
\medskip
\end{center}

\section{Other Related Work}\label{sec:related}
In this section, we briefly review the key ideas of attacking neural network models
and existing defense methods based on adversarial training.

\paragraph{Attack methods.}
% . Adversarial attacks have been extensively studied in the recent
% years. One of the baseline attacks to deep nerual networks is the Fast Gradient Sign Method (FGSM) [GSS15].
% FGSM computes an adversarial example as
Recent years have seen adversarial attack studied extensively.
The proposed attack methods fall under two general categories, 
\emph{white-box} and \emph{black-box} attacks.
% Adversarial attacks have been broadly studied in recent years, which can be categorized into two
% main class, white box attack and black box attack. 

The white-box threat model assumes that the attacker knows
the model's structure and parameters fully.
This means that the attacker can exploit the model's gradient (with respect to the input)
to find adversarial examples.
A baseline of such attacks is the
\textit{Fast Gradient Sign Method} (FGSM)~\citep{goodfellow2014explaining},
which constructs the adversarial example $\bm{x}^\prime$ of a given labeled data $(\bm{x}, y)$
using a gradient-based rule:
\begin{equation}
    \bm{x}^\prime = \bm{x} + \epsilon \text{sign}(\nabla_{\bm{x}} L(f(\bm{x}), y)),
\end{equation}
where $f(x)$ denotes the neural network model's output,
$L(\cdot)$ is the loss function provided $f(x)$ and input label $y$, 
and $\epsilon$ is the perturbation range for the allowed adversarial example.
%  is one of the baseline attacks in white box setting,
%  which can construct adversarial sample in one gradient descent iteration.
%  FGSM compute adversarial samples $x^\prime$ use following update rules:
%  \begin{align*}
%      x^\prime = x + \epsilon \text{sign}(\nabla_x L(f(x), y)),
%  \end{align*}
%  where $x$ is the input sample, $f(x)$ denotes the neural network's output,
%  $L(\cdot)$ is the loss function for $f(x)$ and input label $y$, 
%  and $\epsilon$ is the perturbation size.

Extending FGSM, \textit{Projected Gradient Descent} (PGD)~\citep{kurakin2016adversarial} 
utilizes local first-order gradient of the network in a multi-step fashion,
and is considered the ``strongest'' first-order adversary~\citep{madry2017towards}.
In each step of PGD, the adversary example is updated by a FGSM rule, namely,
% which is more powerful and can be seen as an universal first-order "first-order
% adversary", i.e., "the strongest attack utilizing the local first order
% information about the network"~\cite{madry2017towards}.  It is basically an
% iterative version of FGSM:
\begin{equation}
    \bm{x}^\prime_{n+1} = \Pi_{\bm{x^\prime} \in \Delta_{\epsilon}} \bm{x}^\prime_{n} + \epsilon \text{sign}(\nabla_{\bm{x}} L(f(\bm{x}^\prime_{n}), y)),
\end{equation}
where $\bm{x}^\prime_{n}$ is the adversary examples after $n$ steps
and $\Pi_{\bm{x} \in \Delta_{\epsilon}}(\bm{x}^\prime_n)$ 
projects $\bm{x}^\prime_{n}$ back into an allowed perturbation range $\Delta_{\epsilon}$
(such as an $\epsilon$ ball of $\bm{x}$ under certain distance measure).
% \rev{
Other attacks include Deepfool~\citep{moosavi2016deepfool}, C\&W~\citep{carlini2017towards} and momentum-based attack~\citep{dong2018boosting}.
Those methods are all using first-order gradient information to construct adversarial samples.
% }

% \rev{
% while the black-box attack threat model assumes the attacker has no information about the
% model’s architecture or parameters. In some black-box attack threat model, it allows the attacker
% send queries to the model to gather more information.

% The threat model of black-box attack is a strict subset of the threat model white-box attack.
% The most successful black-box attack is transfer attack~\cite{papernot2017practical, tramer2017ensemble},
% which first construct adversarial samples on an adversarial trained network and then attack
% the black-box model use these samples. There is also exist some gradient-free black-box attack,
% including boundary attack~\cite{brendel2017decision, chen2019boundary}, one-pixel attack~\cite{su2019one} and local search attack~\cite{narodytska2016simple}.
% Those methods are heavily rely on repeatedly evaluating the model and are less effective than
% gradient-based white-box attack.
% }

The black-box threat model is a strict subset of the white-box threat model. It assumes that 
the attacker has no information about the
model’s architecture or parameters. Some black-box attack model allows the attacker 
to query the victim neural network to gather (or reverse-engineer) information.
By far the most successful black-box attack is transfer attack~\citep{papernot2017practical, tramer2017ensemble}.
The idea is to first construct adversarial examples on an adversarially trained network and then attack
the black-box network model use these samples. There also exist some
gradient-free black-box attack methods, such as boundary
attack~\citep{brendel2017decision, chen2019boundary}, one-pixel
attack~\citep{su2019one} and local search attack~\citep{narodytska2016simple}.
Those methods rely on repeatedly evaluating the model and are not as
effective as gradient-based white-box attacks.

\paragraph{Adversarial training.}
Adversarial training
\citep{goodfellow2014explaining,madry2017towards,kurakin2016adversarial,huang2015learning}
is by far the most successful method against adversarial attacks.  
It trains the network model with adversarial images generated during the training
time.  
~\citet{madry2017towards} showed that adversarial training 
in essence solves the following min-max optimization problem:
\begin{equation}
    \min_{f} \mathbb{E}\{ \max_{\bm{x}^\prime \in \Delta_{\epsilon}} L(f(\bm{x}^\prime), y) \},
\end{equation}
where $\Delta_{\epsilon}$ is the set of allowed perturbations of training samples, and $y$ denotes
the true label of each training sample. 
Recent works that achieve state-of-the-art adversarial
robustness rely on adversarial training ~\citep{zhang2019theoretically,
xie2018feature}.  However, adversarial training is notoriously slow because it
requires finding adversarial samples on-the-fly at each training epoch. 
Its prohibitive cost makes adversarial training difficult to scale 
to large datasets such as ImageNet \citep{deng2009imagenet} unless 
enormous computation resources are available.  
Recently, \citet{shafahi2019adversarial} revise the adversarial training algorithm to make it
has similar training time as regular training, 
while keep the standard and robust accuracy comparable to standard adversarial training.
% In addition,
% adversarially trained network usually can not generalize to defend other attacks that 
% are not appeared in the training. For example, a model trained with $\ell_\infty$ adversarial
% samples is not robust against $\ell_0$ adversarial attacks. 
%%% Our method, in contrast, is able to reduce the success rate for a wide range of gradient-based attacks.

\paragraph{Regularization.}
Another type of defense is based on regularizing the neural network, and many works
of this type are combined with adversarial training. For example,
\textit{feature denoising}~\citep{xie2018feature}
%% achieves state-of-the-art results on ImageNet by
adds several denoise blocks to the network structure and trains the network with adversarial training.
Zhang et al.~\citep{zhang2019theoretically} 
explicitly added a regularization term to balance the trade-off between standard accuracy and robustness,
obtaining state-of-the-art robust accuracy on CIFAR. 

Some other regularization-based methods require no adversarial training.
For example,
\citet{inputgradient} proposed to regularize the gradient of the
model with respect to its input; ~\citet{Zheng_2016_CVPR} generated adversarial samples
by adding random Gaussian noise to input data. 
However, these methods are shown to be brittle under stronger iterative
gradient-based attacks such as PGD~\citep{zhang2019theoretically}. 
In contrast, as demonstrated in our experiments, 
our method without using adversarial training is able to greatly improve 
robustness under PGD and other attacks.

\section{Discontinuity Jump of $\sfW\phi_k({x})$}\label{sec:jmp}
Consider a gradual and smooth change of the vector $\bm{x}$.
For the ease of illustration, let us assume all the values in $\bm{x}$ are distinct.
Because every element in $\bm{x}$ changes smoothly, when the activation
pattern $\mathcal{A}(\bm{x})$ changes, the $k$-th largest and $k+1$-th largest value
in $\bm{x}$ must swap: the previously $k$-th largest value is removed from the 
activation pattern, while the previously $k+1$-th largest value is added
in the activation pattern.
Let $i$ and $j$ denote the indices of the two values, that is,
$x_i$ is previously the $k$-th largest and $x_j$ is previously the $k+1$-th largest.
When this swap happens, $x_i$ and $x_j$ must be infinitesimally close to each other,
and we use $x^*$ to indicate their common value.

This swap affects the computation of $\sfW\phi_k(\bm{x})$. Before the swap happens,
$x_i$ is in the activation pattern but $x_j$ is not, therefore $\sfW_i$ takes effect
but $\sfW_j$ does not. After the swap, $\sfW_j$ takes effect while $\sfW_j$ is suppressed.
Therefore, the discontinuity jump due to this swap is $(\sfW_j - \sfW_i)x^*$.

When $\sfW$ is determined, the magnitude of the jump depends on $x^*$.
Recall that $x^*$ is the $k$-th largest value in $\bm{x}$ when the swap happens.
Thus, it depends on $k$ and in turn the sparsity ratio $\gamma$: the smaller the $\gamma$
is, the smaller $k$ is effectively used (for a fixed vector length).
As a result, the $k$-th largest value becomes larger---when $k=1$, the largest value of
$\bm{x}$ is used as $x^*$.

\section{Theoretical Proofs}\label{sec:proofs}

In this section, we will prove Theorem~\ref{thm:dense} and Theorem~\ref{thm:the_second_main_thm}. 
The formal version of the two theorems are Theorem~\ref{thm:formal1} and Theorem~\ref{lem:uniqueness_lemma} respectively.

\paragraph{Notation.} We use $[n]$ to denote the set $\{1,2,\cdots, n\}$. 
We use $\mathbf{1}(\mathcal{E})$ to indicate an indicator variable. 
If the event $\mathcal{E}$ happens, the value of $\mathbf{1}(\mathcal{E})$ is $1$.
Otherwise the value of $\mathbf{1}(\mathcal{E})$ is $0$.
For a weight matrix $W$, we use $W_i$ to denote the $i$-th row of $W$. 
For a bias vector $b$, we use $b_i$ to denote the $i$-th entry of $b$.

In this section, we show some behaviors of the $k$-WTA activation function. 
Recall that an $n$-layer neural network $f(x)$ with $k$-WTA activation function can be seen as the following:
\begin{align*}
f(x)=W^{(1)} \cdot \phi_k(W^{(2)} \cdot \phi_k ( \cdots  \phi_k(W^{(n)} x + b^{(n)}) ) + b^{(2)}) + b^{(1)}
\end{align*}
where $W^{(i)}$ is the weight matrix, $b^{(i)}$ is the bias vector of the $i$-th layer, and $\phi(\cdot)$ is the $k$-WTA activation function, i.e., for an arbitrary vector $y$, $\phi_k(y)$ is defined as the following:
\begin{align*}
\phi_k(y)_j=\left\{\begin{array}{ll}y_j, & \text{if $y_j$ is one of the top-$k$ largest values,} \\ 0, & \text{otherwise.}\end{array}\right.
\end{align*}
For simplicity of the notation, if $k$ is clear in the context, we will just use $\phi(y)$ for short.
Notice that if there is a tie in the above definition, we assume the entry with smaller index has larger value.
For a vector $y\in\mathbb{R}^l$, we define the activation pattern $\mathcal{A}(y)\subseteq[l]$ as 
\begin{align*}
\mathcal{A}(y)=\{i\in [l]\mid \text{$y_i$ is one of the top-$k$ largest values}\}.
\end{align*}

Notice that if the activation pattern $\mathcal{A}(y)$ is different from $\mathcal{A}(y')$, then $W\cdot \phi(y)$ and $W\cdot \phi(y')$ will be in different linear region. 
Actually, $W\cdot \phi(y)$ may even represent a discontinuous function.
In the next section, we will show that when the network is much wider, the function may be more discontinuous with respect to the input.

\subsection{Discontinuity with Respect to the Input} 

We only consider the activation pattern of the output of one layer. 
We consider the behavior of the network after the initialization of the weight matrix and the bias vector.
By initialization, the entries of the weight matrix $W$ are i.i.d. random Gaussian variables, and the bias vector is zero. 
We can show that if the weight matrix is sufficiently wide, then for any vector $x$, with high probability, for all vector $x'$ satisfying that the "perpendicular'' distance between $x$ and $x'$ is larger than a small threshold, the activation patterns of $Wx$ and $Wx'$ are different.

Notice that the scaling of $W$ does not change the activation pattern of $Wx$ for any $x$, we can thus assume that each entry of $W$ is a random variable with standard Gaussian distribution $N(0,1)$.

Before we prove Theorem~\ref{thm:formal1}, let us prove several useful lemmas.
The following several lemmas does not depend on the randomness of the weight matrix.
\begin{lemma}[Inputs with the same activation pattern form a convex set]\label{lem:region_is_convex}
Given an arbitrary weight matrix $W\in\mathbb{R}^{l\times m}$ and an arbitrary bias vector $b\in\mathbb{R}^l$, for any $x\in\mathbb{R}^m$, the set of all the vectors $x'\in\mathbb{R}^m$ satisfying $\mathcal{A}(Wx'+b)=\mathcal{A}(Wx + b)$ is convex, i.e., the set
\begin{align*}
\{x'\in\mathbb{R}^m \mid \mathcal{A}(Wx + b)=\mathcal{A}(Wx' + b) \}
\end{align*}
is convex.
\end{lemma}
\begin{proof}
%Suppose the activation pattern $\mathcal{A}(Wx+b)=\{i_1,i_2,\cdots,i_k\}\subseteq[l]$.

If $\mathcal{A}(Wx'+b)=\mathcal{A}(Wx+b)$, then $x'$ should satisfy:
\begin{align*}
\forall i\in\mathcal{A}(Wx+b),j\in[l]\setminus\mathcal{A}(Wx+b), W_i x' + b_i \geq (\text{or }>)\ W_j x' + b_j.
\end{align*}
Notice that the inequality $W_i x' + b_i \geq (\text{or }>)\ W_j x' + b_j$ denotes a half hyperplane $(W_i-W_j) x' + (b_i-b_j)\geq (\text{or }>)\ 0$.
Thus, the set $\{x'\in\mathbb{R}^m \mid \mathcal{A}(Wx + b)=\mathcal{A}(Wx' + b) \}$ is convex since it is an intersection of half hyperplanes.
\end{proof}

\begin{lemma}[Different patterns of input points with small angle imply different patterns of input points with large angle]\label{lem:small_implies_large}
Let $\alpha\in(0,1)$.
Given an arbitrary weight matrix $W\in\mathbb{R}^{l\times m}$, a bias vector $b=0$, and a vector $x\in\mathbb{R}^m$ with $\|x\|_2=1$, if every vector $x'\in\mathbb{R}^m$ with $\|x'\|_2=1$ and $\langle x,x'\rangle= \alpha$ satisfies $\mathcal{A}(Wx+b)\not=\mathcal{A}(Wx'+b)$, then for any $x''\in\mathbb{R}^m$ with $\|x''\|_2=1$ and $\langle x,x''\rangle < \alpha$, it satisfies $\mathcal{A}(Wx+b)\not=\mathcal{A}(Wx''+b)$.
\end{lemma}
\begin{proof}
We draw a line between $x$ and $x''$. There must be a point $x^*\in\mathbb{R}^m$ on the line and $\langle x,x' \rangle = \alpha$, where $x'=x^*/\|x^*\|_2$.
Since $b=0$, we have $\mathcal{A}(Wx^*+b)=\mathcal{A}(Wx'+b)\not=\mathcal{A}(Wx+b)$.
Since $x^*$ is on the line between $x$ and $x''$, we have $\mathcal{A}(Wx''+b)\not=\mathcal{A}(Wx+b)$ by convexity (see Lemma~\ref{lem:region_is_convex}).
\end{proof}

\begin{lemma}[A sufficient condition for different patterns]\label{lem:sufficient_different}
Consider two vectors $y\in\mathbb{R}^l$ and $y'\in\mathbb{R}^l$.
If $\exists i\in\mathcal{A}(y),j\in[l]\setminus\mathcal{A}(y)$ such that $y'_i<y'_j$, then $\mathcal{A}(y)\not=\mathcal{A}(y')$.
\end{lemma}
\begin{proof}
Suppose $\mathcal{A}(y)=\mathcal{A}(y')$. 
We have $i\in\mathcal{A}(y')$. It means that $y'_i$ is one of the top-$k$ largest values among all entries of $y'$.
Thus $y'_j$ is also one of the top-$k$ largest values, and $j$ should be in $\mathcal{A}(y')$ which leads to a contradiction.
\end{proof}

In the remaining parts, we will assume that each entry of the weight matrix $W\in\mathbb{R}^{l\times m}$ is a standard random Gaussian variable.

\begin{lemma}[Upper bound of the entires of $W$]\label{lem:ub_of_W}
Consider a matrix $W\in\mathbb{R}^{l\times m}$ where each entry is a random variable with standard Gaussian distribution $N(0,1)$.
With probability at least $0.99$, $\forall i\in[l]$, $\|W_i\|_2\leq 10\sqrt{ml}$. 
\end{lemma}
\begin{proof}
Consider a fixed $i\in[l].$
We have $\E[\|W_i\|_2^2] = m$.
By Markov's inequality, we have $\Pr[\|W_i\|_2^2>100ml]\leq 0.01/l$.
By taking union bound over all $i\in[l]$, with probability at least $0.99$, we have $\forall i\in[l],\|W_i\|_2\leq 10\sqrt{ml}$.
\end{proof}

\begin{lemma}[Two vectors may have different activation patterns with a good probability]\label{lem:main_lemma}
Consider a matrix $W\in\mathbb{R}^{l\times m}$ where each entry is a random variable with standard Gaussian distribution $N(0,1)$. 
Let $\gamma\in(0,0.48)$ be the sparsity ratio of the activation, i.e., $\gamma=k/l$.
For any two vectors $x, x'\in\mathbb{R}^m$ with $\|x\|_2=\|x'\|_2=1$ and $\langle x,x'\rangle =\alpha$ for some arbitrary $\alpha\in (0.5,1)$, %if $l>?$, where $\gamma\in(0,0.48)$ is the sparsity $k/l$ of the activation, then 
with probability at least $1-2^{-\Theta((1/\alpha^2-1)\gamma l)}$, $\mathcal{A}(Wx)\not=\mathcal{A}(Wx')$ and $\exists i\in\mathcal{A}(Wx),j\in[l]\setminus \mathcal{A}(Wx)$ such that 
\begin{align*}
W_i x' < W_jx' -\frac{\sqrt{1-\alpha^2}}{24\alpha}\cdot \sqrt{2\pi}.
\end{align*}
\end{lemma}
\begin{proof}
Consider arbitrary two vectors $x,x'\in\mathbb{R}^m$ with $\|x\|_2=\|x'\|_2=1$ and $\langle x,x'\rangle = \alpha$.
We can find an orthogonal matrix $Q\in\mathbb{R}^{m\times m}$ such that $\tilde{x}:=Qx=(1,0,0,\cdots,0)^{\top}\in\mathbb{R}^m$ and $\tilde{x}':=Qx'=(\alpha,\sqrt{1-\alpha^2},0,0,\cdots,0)^{\top}\in\mathbb{R}^m$.
Let $\tilde{W}=WQ^{\top}$.
Then we have $\tilde{W}\tilde{x}=Wx$ and $\tilde{W}\tilde{x}'=Wx'$.
Thus, we only need to analyze the activation patterns of $\tilde{W}\tilde{x}$ and $\tilde{W}\tilde{x}'$.
Since $Q^{\top}$ is an orthogonal matrix and each entry of $W$ is an i.i.d. random variable with standard Gaussian distribution $N(0,1)$, $\tilde{W}=WQ^{\top}$ is also a random matrix where each entry is an i.i.d. random variable with standard Gaussian distribution $N(0,1)$.
Let the entries in the first column of $\tilde{W}$ be $X_1,X_2,\cdots,X_l$ and let the entries in the second column of $\tilde{W}$ be $Y_1,Y_2,\cdots,Y_l$.
Then we have
\begin{align}\label{eq:WxWxprime}
\begin{array}{cc}
Wx=\tilde{W}\tilde{x}=\left(\begin{matrix} X_1\\ X_2\\ \cdots \\ X_l\end{matrix}\right), 
&
Wx'=\tilde{W}\tilde{x'}=\left(\begin{matrix} \alpha X_1 + \sqrt{1-\alpha^2} Y_1\\ \alpha X_2 + \sqrt{1-\alpha^2} Y_2\\ \cdots \\ \alpha X_l + \sqrt{1-\alpha^2} Y_l\end{matrix}\right).
\end{array}
\end{align}
We set $\varepsilon = \sqrt{1-\alpha^2}/(96\alpha)$ and define $R_1'<R_1<R_2<R_2'$ as follows:
\begin{align}
&\Pr_{X\sim N(0,1)}[X\geq R_2']=(1-2\varepsilon)\gamma, \label{eq:R2prime}\\
&\Pr_{X\sim N(0,1)}[X\geq R_2] = (1-\varepsilon)\gamma, \label{eq:R2}\\
&\Pr_{X\sim N(0,1)}[X\geq R_1] = (1+\varepsilon)\gamma, \label{eq:R1}\\
&\Pr_{X\sim N(0,1)}[X\geq R_1'] = (1+2\varepsilon)\gamma. \label{eq:R1prime}
\end{align}
Since $\gamma<0.48$ and $\varepsilon\leq 0.02$, we have $(1+2\varepsilon)\gamma < 0.5$.
It implies $0<R_1'<R_1<R_2<R_2'$.
\begin{claim}\label{cla:rangeofR}
\begin{align*}
R_2'-R_1'\leq 8\varepsilon \sqrt{2\pi}.
\end{align*}
\end{claim}
\begin{proof}
By Equation~\eqref{eq:R2prime} and Equation~\eqref{eq:R1prime},
\begin{align*}
\Pr_{X\sim N(0,1)}[R_1'\leq X \leq R_2' ] = 4\varepsilon\gamma.
\end{align*}
Due to the density function of standard Gaussian distribution, we have
\begin{align*}
\frac{1}{\sqrt{2\pi}} \int_{R_1'}^{R_2'} e^{-t^2/2} \mathrm{d}t = \Pr_{X\sim N(0,1)}[R_1'\leq X \leq R_2' ] = 4\varepsilon \gamma.
\end{align*}
Since $R_2'\geq R_1'\geq 0$, we have $\forall t\in[R_1', R_2'],\ e^{-t^2/2}\geq e^{-R_2'^2/2}$.
Thus,
\begin{align*}
\frac{1}{\sqrt{2\pi}} \cdot e^{-R_2'^2/2} (R_2'-R_1')=\frac{1}{\sqrt{2\pi}} \cdot e^{-R_2'^2/2} \int_{R_1'}^{R_2'} 1 \mathrm{d}t \leq \frac{1}{\sqrt{2\pi}} \int_{R_1'}^{R_2'} e^{-t^2/2} \mathrm{d}t = 4\varepsilon \gamma.
\end{align*}
By the tail bound of Gaussian distribution, we have
\begin{align*}
\Pr_{X\sim N(0,1)}[ X \geq R_2'] \leq e^{-R_2'^2/2}.
\end{align*}
By combining with Equation~\eqref{eq:R2prime}, we have
\begin{align*}
&(1-2\varepsilon)\gamma \cdot \frac{1}{\sqrt{2\pi}} (R_2'-R_1')\\
=~&\Pr_{X\sim N(0,1)}[ X \geq R_2'] \cdot \frac{1}{\sqrt{2\pi}} (R_2'-R_1') \\
\leq~ &  e^{-R_2'^2/2}\cdot \frac{1}{\sqrt{2\pi}}  (R_2'-R_1')\\
\leq~ & 4\varepsilon \gamma,
\end{align*}
which implies 
\begin{align*}
R_2'-R_1' \leq \frac{4\varepsilon}{1-2\varepsilon} \sqrt{2\pi}\leq 8\varepsilon \sqrt{2\pi},
\end{align*}
where the last inequality follows from $1-2\varepsilon\geq 0.5$.
\end{proof}

\begin{claim}
\begin{align}
&\Pr_{X_1,X_2,\cdots,X_l}\left[\sum_{i=1}^l \mathbf{1}(X_i\geq R_2) \geq (1-\varepsilon/2)\gamma l\right]\leq e^{-\varepsilon^2\gamma l/24}\label{eq:top_k_range}\\
&\Pr_{X_1,X_2,\cdots,X_l}\left[\sum_{i=1}^l \mathbf{1}(X_i\geq R_1) \leq (1+\varepsilon/2)\gamma l\right]\leq e^{-\varepsilon^2\gamma l/18}\label{eq:non_top_k_range}\\
&\Pr_{X_1,X_2,\cdots,X_l}\left[\sum_{i=1}^l \mathbf{1}(R_2'\geq X_i\geq R_2) \leq \varepsilon\gamma l /2\right] \leq e^{-\varepsilon\gamma l/8}\label{eq:certain_top_k}\\
&\Pr_{X_1,X_2,\cdots,X_l}\left[\sum_{i=1}^l \mathbf{1}(R_1\geq X_i\geq R_1') \leq \varepsilon\gamma l /2\right] \leq e^{-\varepsilon\gamma l/8}\label{eq:certain_non_top_k}
\end{align}
\end{claim}
\begin{proof}
For $i\in[l]$, we have $\E[\mathbf{1}(X_i\geq R_2)]=\Pr[X_i\geq R_2]=(1-\varepsilon)\gamma $ by Equation~\eqref{eq:R2}.
By Chernoff bound, we have 
\begin{align*}
\Pr\left[\sum_{i=1}^l \mathbf{1}(X_i\geq R_2)\geq (1+\varepsilon/2)\cdot(1-\varepsilon)\gamma l\right]\leq e^{-(\varepsilon/2)^2(1-2\varepsilon)\gamma l/3}.
\end{align*}
Since $\varepsilon\leq 0.02$,
\begin{align*}
\Pr\left[\sum_{i=1}^l \mathbf{1}(X_i\geq R_2) \geq (1-\varepsilon/2)\gamma l\right]\leq e^{-\varepsilon^2\gamma l/24}.
\end{align*}
We have $\E[\mathbf{1}(X_i\geq R_1)]=\Pr[X_i\geq R_1]=(1+\varepsilon)\gamma $ by Equation~\eqref{eq:R1}.
By Chernoff bound, we have
\begin{align*}
\Pr\left[\sum_{i=1}^l \mathbf{1}(X_i\geq R_1) \leq (1-\varepsilon/3)\cdot (1+\varepsilon)\gamma l\right]\leq e^{-(\varepsilon/3)^2(1+\varepsilon)\gamma l /2}.
\end{align*}
Thus,
\begin{align*}
\Pr\left[\sum_{i=1}^l \mathbf{1}(X_i\geq R_i)\leq (1+\varepsilon/2)\gamma l\right]\leq e^{-\varepsilon^2\gamma l /18}
\end{align*}
We have $\E\left[\mathbf{1}(R_2'\geq X_i\geq R_2)\right]=\Pr[R_2'\geq X_i\geq R_2]=\varepsilon \gamma$ by Equation~\eqref{eq:R2prime} and Equation~\eqref{eq:R2}.
By Chernoff bound, we have
\begin{align*}
\Pr\left[\sum_{i=1}^l \mathbf{1}(R_2'\geq X_i\geq R_2)\leq 1/2\cdot \varepsilon \gamma l\right]\leq e^{-\varepsilon\gamma l/8}
\end{align*}
Similarly, we have $\E[\mathbf{1}(R_1\geq X_i\geq R_1')]=\Pr[R_1\geq X_i\geq R_1']=\varepsilon\gamma$ by Equation~\eqref{eq:R1} and Equation~\eqref{eq:R1prime}.
By chernoff bound, we have
\begin{align*}
\Pr_{X_1,X_2,\cdots,X_l}\left[\sum_{i=1}^l \mathbf{1}(R_1\geq X_i\geq R_1') \leq 1/2\cdot \varepsilon\gamma l \right] \leq e^{-\varepsilon\gamma l/8}
\end{align*} 
\end{proof}
Equation~\eqref{eq:top_k_range} says that, with high probability, $\forall i\in[l]$ with $X_i\geq R_2$, it has $i\in \mathcal{A}(Wx)$.
Equation~\eqref{eq:non_top_k_range} says that, with high probability, $\forall i\in[l]$ with $X_i\leq R_1$, it has $i\not\in \mathcal{A}(Wx)$.
Equation~\eqref{eq:certain_non_top_k} (Equation~\eqref{eq:certain_top_k}) says that, with high probability, there are many $i\in[l]$ such that $W_ix\in [R_1',R_1]$ ($W_ix\in[R_2,R_2']$).

Let $\mathcal{E}=\mathcal{E}_1 \wedge \mathcal{E}_2\wedge \mathcal{E}_3 \wedge \mathcal{E}_4$, where
\begin{itemize}
\item $\mathcal{E}_1$: $\sum_{i=1}^l \mathbf{1}(X_i\geq R_2) \leq (1-\varepsilon/2)\gamma l$,
\item $\mathcal{E}_2$: $\sum_{i=1}^l \mathbf{1}(X_i\geq R_1) \geq  (1+\varepsilon/2)\gamma l$,
\item $\mathcal{E}_3$: $\sum_{i=1}^l \mathbf{1}(R_1\geq X_i\geq R_1') \geq \varepsilon\gamma l /2$,
\item $\mathcal{E}_4$: $\sum_{i=1}^l \mathbf{1}(R_2'\geq X_i\geq R_2) \geq \varepsilon\gamma l /2$.
\end{itemize}

According to Equation~\eqref{eq:top_k_range}, Equation~\eqref{eq:non_top_k_range}, Equation~\eqref{eq:certain_top_k} and Equation~\eqref{eq:certain_non_top_k},
the probability that $\mathcal{E}$ happens is at least 
\begin{align}\label{eq:prob_E}
1-4e^{-\varepsilon^2\gamma l/24}
\end{align} 
by union bound over $\bar{\mathcal{E}_1},\bar{\mathcal{E}_2},\bar{\mathcal{E}_3},\bar{\mathcal{E}_4}$.

\begin{claim}\label{cla:top_k_decrease}
Condition on $\mathcal{E}$, the probability that $\exists i\in [l]$ with $X_i \in [R_2,R_2']$ such that $Y_i<-\alpha/\sqrt{1-\alpha^2}\cdot 16\varepsilon \sqrt{2\pi}$ is at least 
\begin{align*}
1-\left(16\varepsilon\cdot \frac{\alpha}{\sqrt{1-\alpha^2}}+\frac{1}{2}\right)^{\varepsilon\gamma l / 2}.
\end{align*}
\end{claim}
\begin{proof}
For a fixed $i\in [l]$, 
\begin{align*}
\Pr\left[Y_i\geq -\alpha/\sqrt{1-\alpha^2}\cdot 16\varepsilon\sqrt{2\pi}\right]&= \int_{-\alpha/\sqrt{1-\alpha^2}\cdot 16\varepsilon\sqrt{2\pi}}^0 \frac{1}{\sqrt{2\pi}} e^{-t^2/2} \mathrm{d}t + \frac{1}{2} \\
&\leq \frac{1}{\sqrt{2\pi}} \cdot \alpha/\sqrt{1-\alpha^2}\cdot 16\varepsilon\sqrt{2\pi} + \frac{1}{2}\\
& = 16\varepsilon \cdot \frac{\alpha}{\sqrt{1-\alpha^2}} + \frac{1}{2}.
\end{align*}
Thus, according to event $\mathcal{E}_4$, we have
\begin{align*}
&\Pr\left[\forall i\text{ with }X_i\in [R_2,R_2'], Y_i\geq -\alpha/\sqrt{1-\alpha^2}\cdot 16\varepsilon\sqrt{2\pi} \mid \mathcal{E}\right]
\leq \left(16\varepsilon\cdot \frac{\alpha}{\sqrt{1-\alpha^2}}+\frac{1}{2}\right)^{\varepsilon\gamma l / 2}.
\end{align*}
\end{proof}
\begin{claim}
Condition on $\mathcal{E}$, the probability that $\exists i\in[l]$ with $X_i\in[R_1',R_1]$ such that $Y_i\geq 0$ is at least 
$
1-\left(1/2\right)^{\varepsilon\gamma l /2}.
$
\end{claim}\label{cla:non_top_k_increase}
\begin{proof}
For a fixed $i\in[l]$,
$
\Pr[Y_i\leq 0] = 1/2.
$
Thus, according to event $\mathcal{E}_3$, we have 
\begin{align*}
\Pr\left[\forall i\text{ with }X_i\in[R_1',R_1],Y_i\leq 0\mid \mathcal{E}\right]\leq (1/2)^{\varepsilon\gamma l/2}.
\end{align*}
\end{proof}
Condition on that $\mathcal{E}$ happens. 
Because of $\mathcal{E}_1$, if $X_i\geq R_2$, $X_i$ must be one of the top-$k$ largest values.
Due to Equation~\eqref{eq:WxWxprime}, we have $X_i=W_ix$.
Thus, if $X_i\geq R_2$, $i\in\mathcal{A}(Wx)$.
By Claim~\ref{cla:top_k_decrease}, with probability at least 
\begin{align}\label{eq:prob_decrease}
1-\left(16\varepsilon\cdot \frac{\alpha}{\sqrt{1-\alpha^2}}+\frac{1}{2}\right)^{\varepsilon\gamma l/2},
\end{align}
there is $i\in \mathcal{A}(Wx)$ such that 
\begin{align}
W_ix'&=\alpha X_i+\sqrt{1-\alpha^2} Y_i\notag\\
&\leq \alpha X_i + \sqrt{1-\alpha^2} \cdot \left(-\frac{\alpha}{\sqrt{1-\alpha^2}}\cdot 16\varepsilon \sqrt{2\pi}\right)\notag\\
&= \alpha (X_i-16\varepsilon\sqrt{2\pi})\notag\\
&\leq \alpha (R_2'-16\varepsilon\sqrt{2\pi}), \label{eq:Wxprime_decrease}
\end{align}
where the first step follows from Equation~\eqref{eq:WxWxprime}, the second step follows from $Y_i\leq -\alpha/\sqrt{1-\alpha^2}\cdot 16\varepsilon \sqrt{2\pi}$, and the last step follows from $X_i\in[R_2,R_2']$.

Because of $\mathcal{E}_2$ if $X_j\leq R_1$, $X_j$ should not be one of the top-$k$ largest values.
Due to Equation~\eqref{eq:WxWxprime}, we have $X_j=W_jx$. 
Thus, if $X_j\leq R_1$, $j\not\in\mathcal{A}(Wx)$.
By Claim~\ref{cla:non_top_k_increase}, with probability at least 
\begin{align}\label{eq:prob_increase}
1-\left(1/2\right)^{\varepsilon\gamma l/2},
\end{align}
there is $j\not\in\mathcal{A}(Wx)$ such that
\begin{align}
W_jx'&=\alpha X_j + \sqrt{1-\alpha^2} Y_j\geq \alpha X_j\geq \alpha R'_1,\label{eq:Wxprime_increase}
\end{align}
where the first step follows from Equation~\eqref{eq:WxWxprime}, the second step follows from $Y_j\geq 0$, and the last step follows from $X_j\in[R_1',R_1]$.

By Equation~\eqref{eq:Wxprime_increase} and Equation~\eqref{eq:Wxprime_decrease}, $\exists i\in\mathcal{A}(Wx),j\in[l]\setminus \mathcal{A}(Wx)$,
\begin{align*}
W_ix'&\leq \alpha(R_2'-16\varepsilon\sqrt{2\pi})\leq \alpha(R_1'-8\varepsilon\sqrt{2\pi})\leq W_jx'-8\alpha\varepsilon\sqrt{2\pi}\\
&\leq W_jx'-4\varepsilon\sqrt{2\pi}=W_jx'-\frac{\sqrt{1-\alpha^2}}{24\alpha}\cdot \sqrt{2\pi},
\end{align*}
where the second step follows from Claim~\ref{cla:rangeofR}, the forth step follows from $\alpha \geq 0.5$, and the last step follows from $\varepsilon=\sqrt{1-\alpha^2}/(96\alpha)$.
By Lemma~\ref{lem:sufficient_different}, we can conclude $\mathcal{A}(Wx)\not=\mathcal{A}(Wx')$.
By Equation~\eqref{eq:prob_E}, Equation~\eqref{eq:prob_decrease}, Equation~\eqref{eq:prob_increase}, and union bound, the overall probability is at least
\begin{align*}
&1-\left(4e^{-\varepsilon^2\gamma l /24}+\left(16\varepsilon\cdot\frac{\alpha}{\sqrt{1-\alpha^2}}+\frac{1}{2}\right)^{\varepsilon\gamma l /2}+\left(\frac12\right)^{\varepsilon\gamma l/2}\right)\\
\geq~& 1-\left(4e^{-\varepsilon^2\gamma l /24}+\left(\frac{2}{3}\right)^{\varepsilon\gamma l /2}+\left(\frac12\right)^{\varepsilon\gamma l/2}\right)\\
\geq~& 1- 6\cdot \left(\frac{2}{3}\right)^{\varepsilon^2\gamma l /24}\\
\geq~& 1 - 2^{-\Theta\left(\left(\frac{1}{\alpha^2}-1\right)\gamma l\right) },
\end{align*}
where the first and the last step follows from $\varepsilon=\sqrt{1-\alpha^2}/(96\alpha)$
\end{proof}

Next, we will use a tool called $\varepsilon$-net.
\begin{definition}[$\varepsilon$-Net]
For a given set $\mathcal{S}$, if there is a set $\mathcal{N}\subseteq \mathcal{S}$ such that $\forall x\in\mathcal{S}$ there exists a vector $y\in\mathcal{N}$ such that $\|x-y\|_2\leq \varepsilon$, then $\mathcal{N}$ is an $\varepsilon$-net of $\mathcal{S}$.
\end{definition}
There is a standard upper bound of the size of an $\varepsilon$-net of a unit norm ball.
\begin{lemma}[\cite{wojtaszczyk1996banach} II.E, 10]\label{lem:standard_eps_net}
Given a matrix $U\in\mathbb{R}^{m\times d}$, let $\mathcal{S}=\{Uy\mid \|Uy\|_2=1\}$. 
For $\varepsilon\in(0,1)$, there is an $\varepsilon$-net $\mathcal{N}$ of $\mathcal{S}$ with $|\mathcal{N}|\leq (1+1/\varepsilon)^d$.
\end{lemma}

Now we can extend above lemma to the following.
\begin{lemma}[$\varepsilon$-Net for the set of points with a certain angle]\label{lem:new_eps_net}
Given a vector $x\in\mathbb{R}^m$ with $\|x\|_2=1$ and a parameter $\alpha\in(-1,1)$, let $\mathcal{S}=\{x'\in\mathbb{R}^m\mid \|x'\|_2=1,\langle x,x'\rangle=\alpha\}$.
For $\varepsilon\in(0,1)$, there is an $\varepsilon$-net $\mathcal{N}$ of $\mathcal{S}$ with $|\mathcal{N}|\leq (1+1/\varepsilon)^{m-1}$.
\end{lemma}

\begin{proof}
Let $U\in\mathbb{R}^{m\times (m-1)}$ have orthonormal columns and $Ux = 0$.
Then $\mathcal{S}$ can be represented as
\begin{align*}
\mathcal{S}=\{\alpha\cdot x + \sqrt{1-\alpha^2}\cdot Uy\mid y\in\mathbb{R}^{m-1},\|Uy\|_2=1\}.
\end{align*}
Let
\begin{align*} 
\mathcal{S}'=\{Uy\mid y\in\mathbb{R}^{m-1}, \|Uy\|_2=1\}.
\end{align*}
According to Lemma~\ref{lem:standard_eps_net}, there is an $\varepsilon$-net $\mathcal{N}'$ of $\mathcal{S}'$ with size $|\mathcal{N}'|\leq (1+1/\varepsilon)^{m-1}$.
We construct $\mathcal{N}$ as following:
\begin{align*}
\mathcal{N}=\{\alpha\cdot x+\sqrt{1-\alpha^2}\cdot z\mid z\in\mathcal{N}'\}.
\end{align*}
It is obvious that $|\mathcal{N}|=|\mathcal{N}'|\leq (1+1/\varepsilon)^{m-1}$.
Next, we will show that $\mathcal{N}$ is indeed an $\varepsilon$-net of $\mathcal{S}$.
Let $x'$ be an arbitrary vector from $\mathcal{S}$.
Let $x'=\alpha\cdot x + \sqrt{1-\alpha^2} \cdot z$ for some $z\in\mathcal{S}'$.
There is a vector $(\alpha\cdot x + \sqrt{1-\alpha^2} \cdot z')\in\mathcal{N}$ such that $z'\in\mathcal{N}'$ and $\|z-z'\|_2\leq \varepsilon$. 
Thus, we have
\begin{align*}
\|x'-(\alpha\cdot x + \sqrt{1-\alpha^2} \cdot z')\|_2 = \sqrt{1-\alpha^2}\|z-z'\|_2\leq \varepsilon.
\end{align*}
\end{proof}

\begin{theorem}[Rotating a vector a little bit may change the activation pattern]\label{thm:main}
Consider a weight matrix $W\in\mathbb{R}^{l\times m}$ where each entry is an i.i.d. sample drawn from the Gaussian distribution $N(0,1/l)$.
Let $\gamma\in(0,0.48)$ be the sparsity ratio of the activation function, i.e., $\gamma=k/l$. 
With probability at least $0.99$, it has $\forall i\in[l],\|W_i\|_2\leq 10\sqrt{m}$.
Condition on that $\forall i\in[l],\|W_i\|_2\leq 10\sqrt{m}$ happens, then,
for any $x\in\mathbb{R}^m$ and $\alpha\in(0.5,1)$, if 
\begin{align*}
l\geq C\cdot \left(\frac{m+\log(1/\delta)}{\gamma} \cdot \frac{1}{1-\alpha^2}\right)\cdot \log\left(\frac{m+\log(1/\delta)}{\gamma} \cdot \frac{1}{1-\alpha^2}\right) 
\end{align*}
for a sufficiently large constant $C$, with probability at least $1-\delta\cdot 2^{-m}$, $\forall x'\in\mathbb{R}^m$ with $\frac{\langle x,x'\rangle}{\|x\|_2\|x'\|_2} \leq \alpha$, $\mathcal{A}(Wx)\not=\mathcal{A}(Wx')$.
\end{theorem}
\begin{proof}
Notice that the scale of $W$ does not affect the activation pattern of $Wx$ for any $x\in\mathbb{R}^m$.
Thus, we assume that each entry of $W$ is a standard Gaussian random variable in the remaining proof, and we will instead condition on $\forall i\in[l],\|W_i\|_2\leq 10\sqrt{ml}$.
 The scale of $x$ or $x'$ will not affect $\frac{\langle x,x'\rangle}{\|x\|_2\|x'\|_2}$. 
It will not affect the activation pattern either.
Thus, we assume $\|x\|_2=\|x'\|_2=1$.

By Lemma~\ref{lem:ub_of_W}, with probability at least $0.99$, we have $\forall i\in[l],\|W_i\|_2\leq 10\sqrt{ml}$. 

Let
\begin{align*} 
\mathcal{S}=\{y\in\mathbb{R}^m \mid \|y\|_2=1,\langle x,y\rangle=\alpha\}.
\end{align*}
Set 
\begin{align*}
\varepsilon=\frac{\sqrt{2\pi(1-\alpha^2)}}{720\alpha\sqrt{ml}}.
\end{align*}
By Lemma~\ref{lem:new_eps_net}, there is an $\varepsilon$-net $\mathcal{N}$ of $\mathcal{S}$ such that 
\begin{align*}
|\mathcal{N}|\leq \left(1+\frac{720\alpha\sqrt{ml}}{\sqrt{2\pi(1-\alpha^2)}}\right)^m.
\end{align*}
By Lemma~\ref{lem:main_lemma}, for any $y\in \mathcal{N}$, with probability at least 
\begin{align*}
1-2^{-\Theta((1/\alpha^2-1)\gamma l)},
\end{align*}
$\exists i\in\mathcal{A}(Wx),j\in[l]\setminus \mathcal{A}(Wx)$ such that 
\begin{align*}
W_iy<W_jy-\frac{\sqrt{1-\alpha^2}}{24\alpha}\cdot \sqrt{2\pi}.
\end{align*}
By taking union bound over all $y\in\mathcal{N}$, with probability at least
\begin{align*}
&1-|\mathcal{N}|\cdot 2^{-\Theta((1/\alpha^2-1)\gamma l)}\\
\geq~&
1- \left(1+\frac{720\alpha\sqrt{ml}}{\sqrt{2\pi(1-\alpha^2)}}\right)^m2^{-\Theta\left(\left(\frac{1}{\alpha^2}-1\right)\gamma l\right)}\\
\geq~&1- \left(1000\cdot\frac{\sqrt{ml}}{\sqrt{1-\alpha^2}}\right)^m2^{-\Theta\left(\left(\frac{1}{\alpha^2}-1\right)\gamma l\right)}\\
\geq~& 1- \left(1000\cdot\frac{\sqrt{ml}}{\sqrt{1-\alpha^2}}\right)^m2^{-C'\cdot\left(\frac{1}{\alpha^2}-1\right)\gamma \cdot \frac{m+\log(1/\delta)}{\gamma} \cdot \frac{\alpha^2}{1-\alpha^2}\cdot \log\left(\frac{ml}{1-\alpha^2}\right)}& \text{// }C'\text{ is a sufficiently large constant}\\
=~&1- \left(1000\cdot\frac{\sqrt{ml}}{\sqrt{1-\alpha^2}}\right)^m2^{-C'\ \cdot (m+\log(1/\delta) ) \cdot  \log\left(\frac{ml}{1-\alpha^2}\right)}\\
\geq~& 1-\delta\cdot 2^{-m},
\end{align*}
the following event $\mathcal{E}'$ happens:
$\forall y\in\mathcal{N},\exists i\in\mathcal{A}(Wx),j\in[l]\setminus \mathcal{A}(Wx)$ such that 
\begin{align*}
W_iy<W_jy-\frac{\sqrt{1-\alpha^2}}{24\alpha}\cdot \sqrt{2\pi}.
\end{align*}
In the remaining of the proof, we will condition on the event $\mathcal{E}'$.
Consider $y'\in\mathcal{S}$.
Since $\mathcal{N}$ is an $\varepsilon$-net of $\mathcal{S}$,
we can always find a $y\in\mathcal{N}$ such that 
\begin{align*}
\|y-y'\|_2\leq \varepsilon =\frac{\sqrt{2\pi(1-\alpha^2)}}{720\alpha\sqrt{ml}}. 
\end{align*}
Since event $\mathcal{E}'$ happens, we can find $i\in\mathcal{A}(Wx)$ and $j\in[l]\setminus \mathcal{A}(Wx)$ such that 
\begin{align*}
W_iy<W_jy-\frac{\sqrt{1-\alpha^2}}{24\alpha}\cdot \sqrt{2\pi}.
\end{align*}
Then, we have
\begin{align*}
W_iy'& = W_iy + W_i(y'-y)\\
&< W_jy- \frac{\sqrt{1-\alpha^2}}{24\alpha}\cdot \sqrt{2\pi} + \|W_i\|_2\|y'-y\|_2\\
&\leq W_jy - \frac{\sqrt{1-\alpha^2}}{24\alpha}\cdot \sqrt{2\pi} + 10\sqrt{ml} \cdot \frac{\sqrt{2\pi(1-\alpha^2)}}{720\alpha\sqrt{ml}}\\
&=W_j y - \frac{\sqrt{1-\alpha^2}}{36\alpha}\cdot \sqrt{2\pi}\\
&=W_j y' + W_j(y-y') - \frac{\sqrt{1-\alpha^2}}{36\alpha}\cdot \sqrt{2\pi}\\
&\leq W_j y' + \|W_j\|_2\|y-y'\|_2- \frac{\sqrt{1-\alpha^2}}{36\alpha}\cdot \sqrt{2\pi}\\
&\leq W_j y' + 10\sqrt{ml} \cdot \frac{\sqrt{2\pi(1-\alpha^2)}}{720\alpha\sqrt{ml}} - \frac{\sqrt{1-\alpha^2}}{36\alpha}\cdot \sqrt{2\pi}\\
&\leq W_jy' - \frac{\sqrt{1-\alpha^2}}{72\alpha}\cdot \sqrt{2\pi},
\end{align*}
where the second step follows from $W_iy<W_jy-\sqrt{1-\alpha^2}/(24\alpha)\cdot\sqrt{2\pi}$ and $W_i(y'-y)\leq \|W_i\|_2\|y'-y\|_2$, the third step follows from $\|W_i\|_2\leq 10\sqrt{ml}$ and $\|y'-y\|_2\leq \sqrt{2\pi(1-\alpha^2)}/(720\alpha\sqrt{ml})$, the sixth step follows from $W_j(y-y')\leq \|W_j\|_2\|y-y'\|_2$, and the seventh step follows from $\|W_i\|_2\leq 10\sqrt{ml}$ and $\|y'-y\|_2\leq \sqrt{2\pi(1-\alpha^2)}/(720\alpha\sqrt{ml})$.

By Lemma~\ref{lem:sufficient_different}, we know that $\mathcal{A}(Wx)\not=\mathcal{A}(Wy')$. 
Thus, $\forall y'\in\mathbb{R}^{m}$ with $\|y'\|_2=1$ and $\langle x,y'\rangle=\alpha$, we have $\mathcal{A}(Wx)\not=\mathcal{A}(Wy')$ conditioned on $\mathcal{E}'$.
By Lemma~\ref{lem:small_implies_large}, we can conclude that $\forall x'\in\mathbb{R}^m$ with $\|x'\|_2=1$ and $\langle x,x'\rangle\leq\alpha$, we have $\mathcal{A}(Wx)\not=\mathcal{A}(Wx')$ conditioned on $\mathcal{E}'$.
\end{proof}

\begin{theorem}[A formal version of Theorem~\ref{thm:dense}]\label{thm:formal1}
Consider a weight matrix $W\in\mathbb{R}^{l\times m}$ where each entry is an i.i.d. sample drawn from the Gaussian distribution $N(0,1/l)$.
Let $\gamma\in(0,0.48)$ be the sparsity ratio of the activation function, i.e., $\gamma=k/l$. 
With probability at least $0.99$, it has $\forall i\in[l],\|W_i\|_2\leq 10\sqrt{m}$.
Condition on that $\forall i\in[l],\|W_i\|_2\leq 10\sqrt{m}$ happens, then,
for any $x\in\mathbb{R}^m$, if 
\begin{align*}
l\geq C\cdot \left(\frac{m+\log(1/\delta)}{\gamma} \cdot \frac{1}{\beta}\right)\cdot \log\left(\frac{m+\log(1/\delta)}{\gamma} \cdot \frac{1}{\beta}\right) 
\end{align*}
for some $\beta\in(0,1)$ and a sufficiently large constant $C$, with probability at least $1-\delta\cdot 2^{-m}$, $\forall x'\in\mathbb{R}^m$ with $\|\Delta x\|_2^2/\|x\|_2^2\geq \beta$, $\mathcal{A}(Wx)\not=\mathcal{A}(Wx')$, where $x'=c\cdot (x + \Delta x)$ for some scaler $c$, and $\Delta x$ is perpendicular to $x$.
\end{theorem}
\begin{proof}
If $\langle x,x'\rangle\leq0$, then the statement follows from Theorem~\ref{thm:main} directly.	
In the following, we consider the case $\langle x,x'\rangle> 0$.
If $\|\Delta x\|_2/\|x\|_2^2\geq \beta$,
\begin{align*}
&\frac{\langle x,x'\rangle^2}{\|x\|_2^2\|x'\|_2^2}\\
=&\frac{c^2\|x\|_2^4}{\|x\|_2^2(c^2(\|x\|_2^2+\|\Delta x\|_2^2))}
=\frac{\|x\|_2^2}{\|x\|_2^2+\|\Delta x\|_2^2}\\
\leq& \frac{\|x\|_2^2}{\|x\|_2^2+\beta\|x\|_2^2} \leq \frac{1}{1+\beta}.
\end{align*}
Thus, we have the bounds:
\begin{align*}
\frac{1}{1-\frac{\langle x,x'\rangle^2}{\|x\|_2^2\|x'\|_2^2}} \leq \frac{1}{\beta} + 1\leq O\left(\frac{1}{\beta}\right).
\end{align*}
By Theorem~\ref{thm:main}, we conclude the proof.
\end{proof}

\begin{example}
Suppose that the training data contains $N$ points $x_1,x_2,\cdots,x_N\in\mathbb{R}^m$ ($m\geq \Omega(\log N)$), where each entry of $x_i$ for $i\in[N]$ is an i.i.d. Bernoulli random variable, i.e., each entry is $1$ with some probability  $p\in(100\log(N)/m,0.5)$ and $0$ otherwise.
Consider a weight matrix $W\in\mathbb{R}^{l\times m}$ where each entry is an i.i.d. sample drawn from the Gaussian distribution $N(0,1/l)$. 
Let $\gamma\in(0,0.48)$ be the sparsity ratio of the activation function, i.e., $\gamma = k/l$. 
If $l\geq \Omega(m/\gamma\cdot\log(m/\gamma))$,
then with probability at least $0.9$, $\forall i,j\in[N]$, the activation pattern of $Wx_i$ and $Wx_j$ are different, i.e., $\mathcal{A}(Wx_i)\not=\mathcal{A}(Wx_j)$.
\end{example}

\begin{proof}
Firstly, let us bound $\|x_i\|_2$.
We have $\E[\|x_i\|_2^2]=\E\left[\sum_{t=1}^m x_{i,t}\right]=pm$.
By Bernstein inequality, we have
\begin{align*}
\Pr\left[\left|\sum_{t=1}^m x_{i,t} - pm\right|>\frac{1}{10}pm\right] \leq 2 e^{-\frac{ (pm/10)^2/2}{pm+\frac13 \cdot \frac1{10} pm}}
\leq 0.01/N.
\end{align*}
Thus, by taking union bound over all $i\in[N]$, with probability at least $0.99$, $\forall i\in[N]$, $\sqrt{0.9pm}\leq \|x_i\|_2\leq \sqrt{1.1pm}$.

Next we consider $\langle x_i,x_j\rangle$. 
Notice that $\E[\langle x_i,x_j\rangle]=\E\left[\sum_{t=1}^m x_{i,t}x_{j,t}\right]=p^2m$.
There are two cases.
\paragraph{Case 1 ($p^2m>20\log N$).}
By Bernstein inequality, we have
\begin{align*}
\Pr\left[\left|\langle x_i,x_j\rangle - p^2m\right|>\frac{1}{2}p^2m\right]\leq 2e^{-\frac{(p^2m/2)^2/2}{p^2m+\frac13 \frac12p^2m}}=2e^{-\frac{3}{28}p^2m}\leq 0.01/N^2.
\end{align*}
By taking union bound over all pairs of $i,j$, with probability at least $0.99$, $\forall i\not =j$, $\langle x_i,x_j\rangle\leq \frac{3}{2}p^2m$.
Since $\|x_i\|_2,\|x_j\|_2\geq \sqrt{0.9pm}$, we have
\begin{align*}
\frac{\langle x_i,x_j\rangle}{\|x_i\|_2\|x_j\|_2}\leq  \frac{3p^2m/2}{0.9pm}=\frac{5}{3}p\leq \frac{5}{6}.
\end{align*}

\paragraph{Case 2 ($p^2m\leq 20\log N$).} By Bernstein inequality, we have
\begin{align*}
\Pr\left[\left|\langle x_i,x_j\rangle - p^2 m\right|> 10 \log N \right]\leq 2e^{-\frac{(10\log N)^2/2}{p^2m+\frac{1}{3} \cdot 10\log N}}\leq 0.01/N^2.
\end{align*}
By taking union bound over all pairs of $i,j$, with probability at least $0.99$, $\forall i\not = j,\langle x_i,x_j \rangle\leq 10\log N$, Since $\|x_i\|_2,\|x_j\|_2\geq \sqrt{0.9pm}\geq \sqrt{90\log N}$, we have
\begin{align*}
\frac{\langle x_i,x_j\rangle}{\|x_i\|_2\|x_j\|_2}\leq \frac{10\log N}{90\log N}= \frac{1}{9}.
\end{align*}

Thus, with probability at least $0.98$, we have $\forall i\not =j$, $\langle x_i, x_j\rangle/(\|x_i\|_2\|x_j\|_2)\leq 5/6$.
By Theorem~\ref{thm:main}, with probability at least $0.99$, $\forall q\in[l],\|W_q\|_2\leq 10\sqrt{m}$.
Condition on this event, and since $\forall i\not=j$ we have $\langle x_i,x_j\rangle/(\|x_i\|_2\|x_j\|_2)\leq 5/6$, by Theorem~\ref{thm:main} again and union bound over all $i\in[N]$, with probability at least $0.99$,
$\forall i\not =j,\mathcal{A}(Wx_i)\not=\mathcal{A}(Wx_j)$.
\end{proof}

\subsection{Disjointness of Activation Patterns of Different Input Points}

Let $X_1,X_2,\cdots,X_m$ be i.i.d. random variables drawn from the standard Gaussian distribution $N(0,1)$.
Let $Z=\sum_{i=1}^m X_i^2$. 
We use the notation $\chi^2_m$ to denote the distribution of $Z$. If $m$ is clear in the context, we just use $\chi^2$ for short.

\begin{lemma}[A property of $\chi^2$ distribution]\label{lem:R_exists}
Let $Z$ be a random variable with $\chi^2_m$ $m$ $(m\geq 2)$  distribution.
Given arbitrary $\varepsilon,\eta\in(0,1)$, if $R$ is sufficiently large then
\begin{align*}
\Pr[Z\geq (1+\varepsilon)R]/\Pr[(1+\varepsilon)R\geq Z\geq R] \leq \eta.
\end{align*}
\end{lemma}

\begin{proof}

Let $R$ be a sufficiently large number such that:
\begin{itemize}
\item $e^{\varepsilon R/2}\geq \frac{4}{\varepsilon}$.
\item $e^{\varepsilon R/8}\geq R^{m/2-1} $.
\item $e^{\varepsilon R/4}\geq \frac{16}{9}\cdot \frac1\eta$.
\end{itemize}
Let $\xi=\varepsilon/4$.
By the density function of $\chi^2$ distribution, we have
\begin{align*}
\Pr[R\leq Z\leq (1+\varepsilon)R] = \frac{1}{2^{m/2}\Gamma(m/2)}\int_{R}^{(1+\varepsilon)R} t^{m/2-1}e^{-t/2} \mathrm{d}t,
\end{align*}
and
\begin{align*}
\Pr[ Z\geq (1+\varepsilon)R] = \frac{1}{2^{m/2}\Gamma(m/2)}\int_{(1+\varepsilon)R}^{\infty} t^{m/2-1}e^{-t/2} \mathrm{d}t,
\end{align*}
where $\Gamma(\cdot)$ is the Gamma function, and for integer $m/2$, $\Gamma(m/2) = (m/2-1)(m/2-2)\cdots\cdot 2\cdot 1=(m/2-1)!$.
By our choice of $R$, we have
\begin{align*}
\Pr[R\leq Z\leq (1+\varepsilon)R]&\geq \frac{1}{2^{m/2}\Gamma(m/2)} \int_{R}^{(1+\varepsilon)R} e^{-t/2}\mathrm{d}t\\
&= \frac{1}{2^{m/2}\Gamma(m/2)} \cdot 2\left(e^{-R/2}-e^{-(1+\varepsilon)R/2}\right)\\
&\geq \frac{1}{2^{m/2}\Gamma(m/2)} \cdot 2(1-\xi)\cdot e^{-R/2},
\end{align*}
where the first step follows from $\forall t\geq R$, $t^{m/2-1}\geq 1$, and the third step follows from 
\begin{align*}
\frac{e^{-(1+\varepsilon)R/2}}{e^{-R/2}} = e^{-\varepsilon R/2} \leq \xi.
\end{align*}
We also have:
\begin{align*}
\Pr[Z\geq (1+\varepsilon)R]&\leq \frac{1}{2^{m/2}\Gamma(m/2)} \int_{(1+\varepsilon)R}^{+\infty} e^{-(1-\xi)t/2} \mathrm{d}t\\
&=\frac{1}{2^{m/2}\Gamma(m/2)}\cdot \frac{2}{1-\xi}\cdot e^{-(1-\xi)(1+\varepsilon)R/2} \\
&\leq \frac{1}{2^{m/2}\Gamma(m/2)}\cdot \frac{2}{1-\xi}\cdot e^{-(1+\varepsilon/2)R/2},
\end{align*}
where the first step follos from $\forall t\geq R,$ $t^{m/2-1}\leq e^{\xi t/2}$, and the third step follows from $(1-\xi)(1+\varepsilon)\geq (1+\varepsilon/2)$.

Thus, we have 
\begin{align*}
\frac{\Pr[Z\geq (1+\varepsilon) R]}{\Pr[(1+\varepsilon)R\geq Z\geq R]} \leq \frac{1}{(1-\xi)^2} e^{-\varepsilon R /4}\leq \frac{16}{9} e^{-\varepsilon R /4}\leq \eta.
\end{align*}
\end{proof}

\begin{lemma}\label{lem:angle_difference}
Consider $x,y,z\in\mathbb{R}^{m}$.
If $\frac{\langle x,y\rangle}{\|x\|_2\|y\|_2}\leq \alpha,\frac{\langle x,z\rangle}{\|x\|_2\|z\|_2}\geq \beta$ for some $\alpha,\beta\geq 0$, then $\frac{\langle y,z\rangle}{\|y\|_2\|z\|_2}\leq \alpha+\sqrt{1-\beta^2}$.
Furthermore, if $\beta=\frac{2+\alpha+\sqrt{2-\alpha^2}}{4}$, then $\frac{\langle y,z\rangle}{\|y\|_2\|z\|_2}\leq (1-\varepsilon_{\alpha}) \beta$, where $\varepsilon_{\alpha}\in(0,1)$ only depends on $\alpha$.
\end{lemma}
\begin{proof}
Without loss of generality, we suppose $\|x\|_2=\|y\|_2=\|z\|_2=1$.
We can decompose $y$ as $a x+y'$ where $y'$ is perpendicular to $x$.
We can decompose $z$ as $b_1 x + b_2 y'/\|y'\|_2 + z'$ where $z'$ is perpendicular to both $x$ and $y'$.
Then we have:
\begin{align*}
\langle y,z\rangle = ab_1 + b_2 \|y'\|_2 \leq \alpha + \sqrt{1-\beta^2},
\end{align*}
where the last inequality follows from $0\leq b_1\leq 1,a\leq \alpha $, and $b_2\leq \sqrt{1-b_1^2}\leq \sqrt{1-\beta^2}$, $0\leq \|y'\|_2\leq 1$.

By solving $\beta\geq \alpha+\sqrt{1-\beta^2}$, we can get $\beta\geq\frac{\alpha+\sqrt{2-\alpha^2}}{2}$.
Thus, if we set
\begin{align*} 
\beta=\frac{1+\frac{\alpha+\sqrt{2-\alpha^2}}{2}}{2},
\end{align*}
$\beta$ should be strictly larger than $\alpha+\sqrt{1-\beta^2}$, and the gap only depends on $\alpha$.
\end{proof}

\begin{lemma}
Give $x\in\mathbb{R}^m$, let $y\in\mathbb{R}^m$ be a random vector, where each entry of $y$ is an i.i.d. sample drawn from the standard Gaussian distribution $N(0,1)$. Given $\beta\in(0.5,1)$, $\Pr[\langle x,y\rangle/(\|x\|_2\|y\|_2)\geq \beta]\geq 1/(1+1/\sqrt{2(1-\beta)})^m$.
\end{lemma}
\begin{proof}
Without loss of generality, we can assume $\|x\|_2=1$.
Let $y'=y/\|y\|_2$.
Since each entry of $y$ is an i.i.d. Gaussian variable, $y'$ is a random vector drawn uniformly from a unit sphere.
Notice that if $\langle x,y'\rangle \geq \beta$, then $\|x-y'\|_2\leq \sqrt{2(1-\beta)}$.
Let $C=\{z\in\mathbb{R}^m\mid \|z\|_2=1,\|z-x\|_2\leq \sqrt{2(1-\beta)}\}$ be a cap, and let $\mathcal{S}=\{z\in\mathbb{R}^m \mid \|z\|_2=1 \}$ be the unit sphere.
Then we have
\begin{align*}
\Pr[\langle x,y'\rangle \geq \beta] = \mathrm{area}(C)/\mathrm{area}(\mathcal{S}).
\end{align*}
According to Lemma~\ref{lem:standard_eps_net}, there is an $\sqrt{2(1-\beta)}$-net $\mathcal{N}$ with $|\mathcal{N}|\leq (1+1/\sqrt{2(1-\beta)})^m$.
If we put a cap centered at each point in $\mathcal{N}$, then the whole unit sphere will be covered. 
Thus, we can conclude
\begin{align*}
\Pr[\langle x,y'\rangle \geq \beta]\geq 1/(1+1/\sqrt{2(1-\beta)})^m.
\end{align*}
\end{proof}

\begin{theorem}[A formal version of Theorem~\ref{thm:the_second_main_thm}]\label{lem:uniqueness_lemma}\label{thm:disjoint}
Consider $N$ data points $x_1,x_2,\cdots,x_N\in\mathbb{R}^m$ and a weight matrix $W\in\mathbb{R}^{l\times m}$ where each entry of $W$ is an i.i.d. sample drawn from the Gaussian distribution $N(0,1/l)$.
Suppose $\forall i\not=j\in[N],$ $\langle x_i,x_j\rangle/(\|x_i\|_2\|x_j\|_2)\leq \alpha$ for some $\alpha\in(0.5,1)$. 
Fix $k\geq 1$ and $\delta\in(0,1)$, if $l$ is sufficiently large, then with probability at least $1-\delta$, 
\begin{align*}
\forall i,j\in[N],\mathcal{A}(Wx_i)\cap \mathcal{A}(Wx_j)=\emptyset.
\end{align*}
\end{theorem}
\begin{proof}
Notice that the scale of $W$ and $x_1,x_2,\cdots,x_N$ do not affect either $\langle x_i,x_j\rangle/(\|x_i\|_2\|x_j\|_2)$ or the activation pattern. 
Thus, we can assume $\|x_1\|_2=\|x_2\|_2=\cdots=\|x_N\|_2=1$ and each entry of $W$ is an i.i.d. standard Gaussian random variable.

Let $\beta = \frac{2+\alpha+\sqrt{2-\alpha^2}}{4}$ and  $\varepsilon_{\alpha}$ be the same as mentioned in Lemma~\ref{lem:angle_difference}.
Set $\varepsilon$ and $\beta'$ as
\begin{align*}
\begin{array}{cc}
\varepsilon=\frac{\frac1\beta-1}{2}, & \beta'=(1+\varepsilon)\beta.
\end{array}
\end{align*}
Now, set
\begin{align*}
\eta=\frac{\delta/100}{100 k\log (N/\delta)\cdot (1+2/\sqrt{2(1-\beta')})^m},
\end{align*}
%and 
%\begin{align*}
%l= \frac{1}{\eta} \cdot \frac{\delta}{100}.
%\end{align*}
and let $R$ satisfies 
\begin{align*}
\Pr_{Z\sim\chi^2_m}[Z\geq (1+\varepsilon)^2 R^2]=\frac{\delta/100}{l}.
\end{align*}
According to Lemma~\ref{lem:R_exists}, if $l$ is sufficiently large, then $R$ is sufficiently large such that
\begin{align*}
\Pr_{Z\sim\chi^2_m}[Z\geq (1+\varepsilon)^2R^2]/\Pr_{Z\sim\chi^2_m}[(1+\varepsilon)^2R^2\geq Z\geq R^2] \leq \eta.
\end{align*}
Notice that for $t\in[l]$, $\|W_t\|_2^2$ is a random variable with $\chi_m^2$ distribution. 
Thus, $\Pr[\|W_t\|_2\geq (1+\varepsilon)R] =\frac{\delta/100}{l}$.
By taking union bound over all $t\in[l]$, with probability at least $1-\delta/100$, $\forall t\in[l]$, $\|W_t\|_2\leq (1+\varepsilon)R$. In the remaining of the proof, we will condition on that $\forall t\in[l],\|W_t\|_2\leq (1+\varepsilon) R$.
Consider $i,j\in[N],t\in[l]$, if $W_tx_i> \beta'R$, then we have
\begin{align*}
\frac{W_t x_i}{\|W_t\|_2}>\frac{\beta'R}{(1+\varepsilon)R}\geq \beta'/(1+\varepsilon)=\beta. 
\end{align*}
Due to Lemma~\ref{lem:angle_difference}, we have
\begin{align*}
\frac{W_tx_j}{\|W_t\|_2} < (1-\varepsilon_{\alpha})\beta.
\end{align*}
Thus, 
\begin{align}\label{eq:other_small}
W_tx_j< (1-\varepsilon_{\alpha})\beta\|W_t\|_2\leq (1-\varepsilon_{\alpha})\beta(1+\varepsilon)R\leq (1-\varepsilon_{\alpha})\beta'R.
\end{align}
Notice that for $i\in[N],t\in[l]$, we have
\begin{align*}
\Pr[W_tx_i>\beta'R]&\geq \Pr[\|W_t\|_2\geq R]\Pr\left[\frac{W_t x_i}{\|W_t\|_2}\geq \beta'\right]\\
& \geq \frac{\delta/100}{l}\cdot \frac{1}{\eta}\cdot \frac{1}{(1+1/\sqrt{2(1-\beta')})^m}\\
& \geq \frac{1}{l}\cdot 100k\log(N/\delta).
\end{align*}
By Chernoff bound, with probability at least $1-\delta/(100N)$, 
\begin{align*}
\sum_{t=1}^l \mathbf{1}(W_tx_i>\beta'R) \geq k.
\end{align*}
By taking union bound over $i\in[N]$, with probability at least $1-\delta/100$, $\forall i\in[N]$, 
 \begin{align*}
 \sum_{t=1}^l \mathbf{1}(W_tx_i>\beta'R) \geq k.
 \end{align*}
This implies that $\forall i\in[N]$, if $t\in\mathcal{A}(Wx_i)$, then $W_tx_i>\beta'R$. 
Due to Equation~\eqref{eq:other_small}, $\forall j\in[N]$, we have $W_tx_j<\beta'R$ which implies that $t\not\in\mathcal{A}(Wx_j)$.
Thus, with probability at least $1-\delta/50\geq 1-\delta$ probability, $\forall i\not=j$, $\mathcal{A}(Wx_i)\cap \mathcal{A}(Wx_j)=\emptyset$.
\end{proof}

\begin{remark}
	Consider any $x_1,x_2,\cdots,x_N\in\mathbb{R}^m$ with $\|x_1\|_2=\|x_2\|_2=\cdots=\|x_N\|_2=1$.
	If $\forall i\not=j\in[N],\langle x_i,x_j\rangle\leq \alpha$ for some $\alpha\in(0.5,1)$, then $|N|\leq (1+2/\sqrt{2(1-\alpha)})^m$.
\end{remark}
\begin{proof}
	Since $\langle x_i,x_j\rangle\leq \alpha$, $\|x_i-x_j\|_2^2=\|x_i\|_2^2+\|x_j\|_2^2-2\langle x_i,x_j\rangle \geq 2-2\alpha$.
	Let $\mathcal{S}$ be the unit sphere, i.e., $\mathcal{S}=\{x\in\mathbb{R}^m\mid \|x\|_2=1\}$.
	Due to Lemma~\ref{lem:standard_eps_net}, there is a $(\sqrt{2(1-\alpha)}/2)$-net $\mathcal{N}$ of $\mathcal{S}$ with size at most $|\mathcal{N}|\leq (1+2/\sqrt{2(1-\alpha)})^m$.
	Consider $x_i,x_j,$ and $y\in\mathcal{N}$.
	By triangle inequality, if $\|x_i-y\|_2<\sqrt{2(1-\alpha)}/2$, then $\|x_j-y\|_2> \sqrt{2(1-\alpha)}/2$ due to $\|x_i-x_j\|_2\geq \sqrt{2(1-\alpha)}$.
	Since $\mathcal{N}$ is a net of $\mathcal{S}$, for each $x_i$, we can find a $y\in\mathcal{N}$ such that $\|x_i-y\|_2<\sqrt{2(1-\alpha)}/2$. 
	Thus, we can conclude $N\leq |\mathcal{N}|\leq (1+2/\sqrt{2(1-\alpha)})^m$.
\end{proof}

\begin{theorem}
Consider $N$ data points $x_1,x_2,\cdots,x_N\in\mathbb{R}^m$ with their corresponding labels $z_1,z_2,\cdots,z_N\in\mathbb{R}$ and a weight matrix $W\in\mathbb{R}^{l\times m}$ where each entry of $W$ is an i.i.d. sample drawn from the Gaussian distribution $N(0,1/l)$.
Suppose $\forall i\not=j\in[N],$ $\langle x_i,x_j\rangle/(\|x_i\|_2\|x_j\|_2)\leq \alpha$ for some $\alpha\in(0.5,1)$. 
Fix $k\geq 1$ and $\delta\in(0,1)$, if $l$ is sufficiently large, then with probability at least $1-\delta$, there exists a vector $v\in\mathbb{R}^{l}$ such that 
\begin{align*}
\forall i\in[N], \langle v, \phi_k(Wx_i)\rangle = z_i. 
\end{align*}
\end{theorem}
\begin{proof}
Due to Theorem~\ref{lem:uniqueness_lemma}, with probability at least $1-\delta$, $\forall i\not =j$, $\mathcal{A}(Wx_i)\cap\mathcal{A}(Wx_j)=\emptyset$.
Let $t_1,t_2,\cdots,t_N\in [l]$ such that $t_i\in\mathcal{A}(Wx_i)$. 
Then $t_i\not\in\mathcal{A}(Wx_j)$ for $j\not =i$.

For each entry $v_t$, if $t=t_i$ for some $i\in[N]$, then set $v_t = z_i/ (W_tx_i)$.
Then for $i\in[N]$, we have
\begin{align*}
\langle v, \phi_k(Wx_i)\rangle
=\sum_{t\in\mathcal{A}(Wx_i)} v_t \cdot W_t x_i  =  z_i/(W_{t_i}x_i) \cdot W_{t_i}x_i= z_i.
\end{align*}
\end{proof}

\section{Additional Experimental Results}
This section presents details of our experiment settings 
and additional results for evaluating
and empirically understanding the robustness of \kwta networks.

\subsection{Experiment Settings}\label{sec:settings}
%% Here we show the omitted experiment settings in \secref{exp}.
First, we describe the details of setting up the experiments described in \secref{exp}.
To compare \kwta networks with their ReLU counterparts, 
we replace all ReLU activations in a network with \kwta activation, while 
retaining all other modules (such as BatchNorm, Convolution, and pooling).
To test on different network architectures, including 
ResNet18, DenseNet121, and Wide ResNet,  we use the 
standard implementations that are publicly available\footnote{https://github.com/kuangliu/pytorch-cifar}. 
All experiments are conducted using \textsf{PyTorch} framework. 

\vspace{-2mm}
\paragraph{Training setups.}
We follow the same training procedure on CIFAR-10 and SVHN datasets.
%%For both CIFAR-10 and SVHN experiment we use the same training protocal.
All the ReLU networks are trained with stochastic gradient descent (SGD) method 
with momentum=$0.9$. We use a learning rate 0.1 from the first to 50-th epoch
and 0.01 from 50-th to 80-th epoch.
To compare with ReLU networks,
the k-WTA networks are trained in the same way as ReLU networks. 
All networks are trained with a batch size of 256.

For \kwta networks with a sparsity ratio $\gamma=0.1$, when adversarial training is
not used, we train them incrementally (recall in \secref{train}). 
starting with $\gamma=0.2$ for 50 epochs with SGD (using learning rate=0.1,
momentum=0.9) and then decreasing $\gamma$ by 0.005 every 2 epochs until
$\gamma$ reaches $0.1$. 

When adversarial training is enabled, we use untargeted PGD attack with 8
iterations to construct adversarial examples.
To train networks with TRADES~\citep{zhang2019theoretically}, we use the implementation of its original
paper\footnote{https://github.com/yaodongyu/TRADES} with 
the parameter $1/\lambda=6$,
% Here we use the parameter $1/\lambda=6$ in TRADES, 
a value that reportedly leads to the best robustness according to the paper.
To train networks with the free adversarial training method~\citep{shafahi2019adversarial},
we implement the training algorithm by following the original paper. We set the parameter
$m=8$ as suggested in the paper.

% For free adversarial training, we use our implementation based on the algorithm
% description in the original paper \citep{shafahi2019adversarial}. We set the
% parameter $m=8$ as their paper suggested.

\vspace{-2mm}
\paragraph{Attack setups.}
All attacks are evaluated under the $\ell_{\infty}$ metric, 
with perturbation size $\epsilon=$ 0.031 (CIFAR-10) and 0.047 (SVHN) 
for pixels ranging in $[0, 1]$. We use Foolbox~\citep{rauber2017foolbox}, a
third-party toolbox for evaluating adversarial robustness.

We use the following setups for generating adversarial examples in various attack methods: 
For PGD attack, we use 40 iterations with random start, the step size is 0.003.
For C\&W attack, we set the binary search step to 5, maximum number of iterations to 20, learning
rate to 0.01, and initial constant to 0.01. For Deepfool, we use 20 steps and
10 sub-samples in its configuration. For momentum attack, we set the step size
to 0.003 and number of iterations to 20.  All other parameters are set by
Foolbox to be its default values.

% --------------------------------------------------------------------------------------

\subsection{Efficacy of Incremental Training}\label{sec:inc}
\begin{figure*}[t]
	\centering
	%\vspace{-1mm}
	\includegraphics[width=1.0\textwidth]{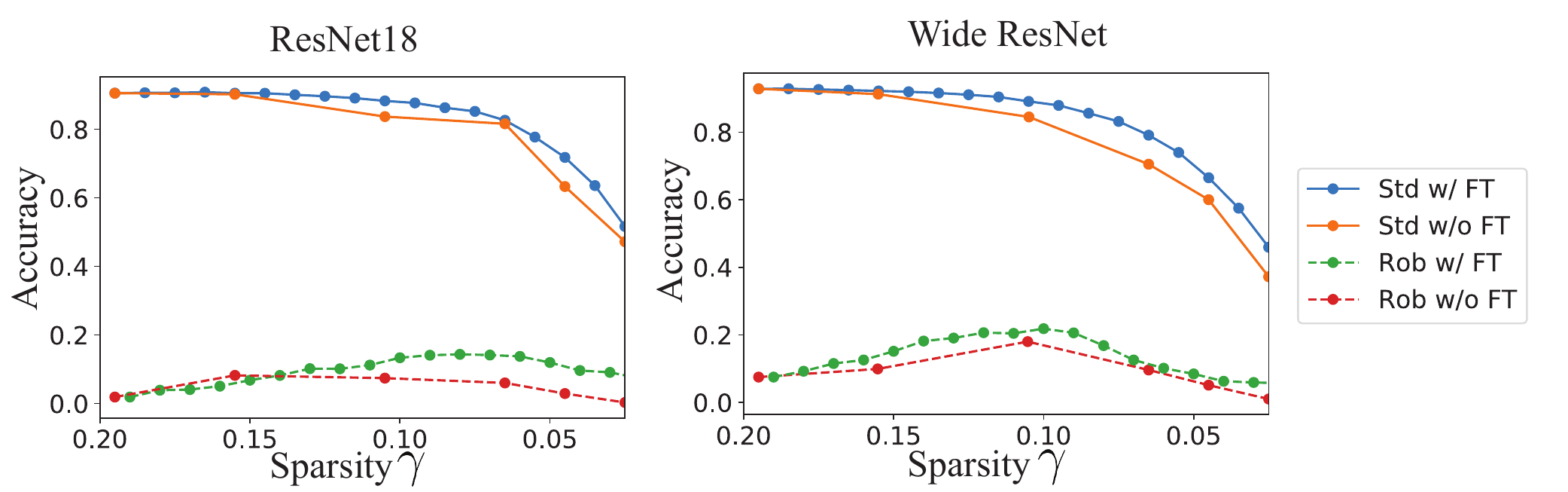}
	\vspace{-5mm}
	\caption{
            \textbf{Efficacy of incremental training.} We sweep through a 
            range of sparsity ratios, and evaluate the standard and robust accuracies
            of two network structures (left: ResNet18 and right: Wide ResNet).
            We compare the performance differences between the regular training (i.e.,
            training without incremental fine-tuning) and the training
            with incremental fine-tuning.
	}\label{fig:ft}
	\vspace{-2mm}
\end{figure*}
We now report additional experiments to demonstrate the efficacy of 
the incremental fine-tuning method (described in \secref{train}).  
As shown in \figref{ft} and described its caption, 
models trained with incremental fine-tuning (denoted
as \textsf{w/ FT} in the plots' legend) performs better in terms of both standard
accuracy (denoted as \textsf{std} in the plots' legend) and robust accuracy (denoted
as \textsf{Rob} in the plots) when the \kwta sparsity $\gamma<0.2$,  suggesting that 
fine-tuning is worthwhile when $\gamma$ is small.

\subsection{Additional results on CIFAR-10}\label{sec:add}
\begin{table}[t]
\centering
\caption{Additional CIFAR-10 results.}\label{tab:cifar_other}
\begin{tabular}{l|l|l|cc}
\bottomrule
Training & Model & Activation & $A_{std}$ & $A_{rob}$     \\ \hline
\multirow{5}{*}{Natural} & \multirow{5}{*}{ResNet-18}  &ReLU & 92.9\% &0.0\% \\
 && LWTA-0.1  & 82.8\% & 3.7\% \\
 && LWTA-0.2  & 84.6\% & 0.9\% \\
 && \kwta-0.1   & 89.3\% & \textbf{13.1\%} \\
 && \kwta-0.2  & 91.7\% &  4.2\%\\ \hline
\multirow{5}{*}{AT} & \multirow{5}{*}{ResNet-18}  &ReLU & 83.5\% &43.6\% \\
 && LWTA-0.1   & 71.4\% & 46.6\% \\
 && LWTA-0.2  & 78.7\% & 43.1\% \\
 && \kwta-0.1   & 78.9\% & \textbf{50.7\%} \\
 && \kwta-0.2  & 81.4\% & 47.4\% \\ \hline
 \multirow{5}{*}{Natural} & \multirow{5}{*}{DenseNet-121}  &ReLU & 93.6\% & 0.0\% \\
 && LWTA-0.1   & 86.1\% & 4.6\% \\
 && LWTA-0.2  & 88.5\% & 1.4\% \\
 && \kwta-0.1   & 90.5\% & \textbf{12.3\%} \\
 && \kwta-0.2  & 93.3\% & 6.2\% \\ \hline
 \multirow{5}{*}{AT} & \multirow{5}{*}{DenseNet-121}  &ReLU & 84.2\% & 46.3\% \\
 && LWTA-0.1   & 74.0\% & 49.1\% \\
 && LWTA-0.2  & 80.2\% & 44.9\% \\
 && \kwta-0.1  & 81.6\% & \textbf{52.4\%}  \\
 && \kwta-0.2  & 83.4\% & 49.6\% \\ \hline
 \multirow{5}{*}{Natural} & \multirow{5}{*}{WideResNet-22-10}  &ReLU &93.4\% & 0.0\% \\
 && LWTA-0.1   & 83.7\% & 4.2\% \\
 && LWTA-0.2  & 86.1\% & 2.8\% \\
 && \kwta-0.1   & 88.6\% & \textbf{18.3\%} \\
 && \kwta-0.2  & 92.7\% & 7.4\% \\ \hline
 \multirow{5}{*}{AT} & \multirow{5}{*}{WideResNet-22-10}  &ReLU &83.3\% &43.1\% \\
 && LWTA-0.1   & 74.2\% & 47.5\% \\
 && LWTA-0.2  & 79.8\% & 44.7\% \\
 && \kwta-0.1   & 78.9\% & \textbf{50.4\%} \\
 && \kwta-0.2  & 82.4\% & 47.1\% \\ \hline
\toprule
\end{tabular}
\vspace{-6mm}
\end{table}

\paragraph{Tests on different network architectures.}
We evaluate the robustness of \kwta on different network architectures, including ResNet-18, DenseNet-121 and WideResNet-22-10. 
The results are reported in \tabref{cifar_other},
where similar to the notation used in \tabref{cifar_svhn_main} of the main text,
$A_{rob}$ is calculated as the worst-case robustness, i.e., under 
the most effective attack among PGD, C\&W, Deepfool and MIM.
The training and attacking settings are same as other experiments described \secref{settings}.  

As shown in \tabref{cifar_other},
while the standard and robustness accuracies, $A_{std}$ and $A_{rob}$, vary
on different network architectures, \kwta networks consistently improves 
the worst-case robustness $A_{rob}$ over ReLU networks, no matter what
the network architecture and training method are used.

% different network architecture have slightly different
% performance in terms of $A_{std}$ and $A_{rob}$, but \kwta can always improve
% the $A_{rob}$ compare to ReLU, regardless of network architectures.

\paragraph{Comparison with LWTA.}
We in addition compare \kwta to LWTA activation~\citep{srivastava2013compete,srivastava2014understanding}.
For fair comparisons, 
we use the same sparsity ratio $\gamma$ in both \kwta and LWTA.
As shown in \tabref{cifar_other}, on all network architectures and training methods we tested, 
\kwta networks consistently have better robustness performance than LWTA networks 
(in terms of both $A_{std}$ and $A_{rob}$). 
These results suggest that \kwta is more suitable then LWTA for defending against adversarial attacks.

\vspace{-2mm}
\paragraph{Transfer attack.}
Since a \kwta network is architecturally similar to its ReLU counterpart---with
the only difference being the activation---we evaluate their robustness under
(black-box) transfer attacks across \kwta and ReLU networks.  To this end, we build a ReLU
and a \kwta-0.1 network on ResNet-18, and train both networks with natural
(non-adversarial) training as well as adversarial training.  This gives us four
different models denoted (in \tabref{cifar_transfer}) as ReLU, \kwta-0.1,
ReLU (AT), and \kwta-0.1 (AT).
We then launch transfer attacks across each pair of models.
We also consider by-far the strongest black-box attack (according to~\citet{papernot2017practical}):
for the same model, for example, a \kwta-0.1 network optimized by adversarial training,
we train two independent versions, each with a different random initialization,
and apply the transfer attacks across the two versions.

The results are reported in \tabref{cifar_transfer}, where each row corresponds to
a target (attacked) model, and each column corresponds to a source model from which the
adversarial examples are generated.
On the diagonal line of \tabref{cifar_transfer}, each entry corresponds to the 
robustness under aforementioned transfer attacks across the two versions of the same models.

%  To evaluate the transferability of \kwta and ReLU networks,
%  we train ReLU and \kwta-0.1 network using ResNet-18, with and without adversarial training.
% For each configuration we train 2 independent copy of it, with different random initialization.

The results suggest that \textbf{1)} 
it is more difficult to transfer attack \kwta networks 
than ReLU networks using adversarial examples from other models, 
and \textbf{2)} it is also more difficult to use adversarial examples of
a \kwta network to attack other models. In a sense, the adversarial examples of a \kwta
network tend to be ``disjoint'' from the adversarial examples of a ReLU network, despite their 
architectural similarity.
Inspecting the diagonal entries of \tabref{cifar_transfer},
we also find that \kwta networks are more robust than their ReLU counterparts 
under the strongest black-box attack~\citep{papernot2017practical}
(i.e., transfer attacks across two different versions of the same model).

% We found that adversarial examples crafted on \kwta is hard to transfer to other models,
% while adversarial examples crated on other models is also difficult to transfer to \kwta models.

% Furthermore, \kwta networks show better performance when use independently
% trained copy of same architecture as source model, which is considered the
% strongest black-box attack so far~\citep{papernot2017practical}.  The detailed
% statistics can be found in \tabref{cifar_transfer}.

\begin{table}[t]
\centering
\caption{Transferability test on CIFAR-10.
}\label{tab:cifar_transfer}
% \hspace{-7mm}
\begin{tabular}{l|cccc}\bottomrule
\multirow{2}{*}{\textbf{Target Model}}& \multicolumn{4}{c}{\textbf{Source Model}}\\ 
&  ReLU & \kwta-0.1 & ReLU (AT) & \kwta-0.1 (AT)  \\ \hline
ReLU & 4.8\% & 75.5\% & 59.4\% & 84.7\%  \\
\kwta-0.1 & 61.2\% & 71.2\% & 67.8\% & 86.4\% \\
ReLU (AT) & 62.7\% & 80.9\% & 61.6\% & 78.6\% \\
\kwta-0.1 (AT)& 79.2\% & 78.6\% & 69.2\% & 67.2\%  \\ \hline 

\end{tabular}

% }

\vspace{-4mm}
\end{table}

\subsection{MNIST Results}\label{sec:mnist}
On MNIST dataset, we conduct experiments with an adversarial perturbation size $\epsilon$=0.3 for
pixels ranging in $[0, 1]$. We use Stochastic Gradient Descent (SGD) with
learning rate=0.01 and momentum=0.9 to train a 3-layer CNN. The training takes 20 epochs for
all the methods we evaluate.  The robust accuracy are evaluated under PGD
attacks that take 20 iterations with random initialization and a step size of 0.03.

The results are summarized in \tabref{mnist_main}.  
Again, \kwta activation consistently improves robustness under all different training methods.
Even with natural (non-adversarial) training, the resulting \kwta network still has 62.2\% robust accuracy,
significantly outperforming ReLU network. 

\begin{table}[t]
\caption{White-box attack results on MNIST dataset.
    \label{tab:mnist_main}}
\centering
% \vspace{-2mm}
\begin{tabular}{l|lcc?l|lcc}
\bottomrule
Activation & Training & $A_{std}$  & $A_{rob}$  & Activation & Training & $A_{std}$  &  $A_{rob}$ \\ \hline
\multirow{4}{*}{ReLU}& Natural & 99.4\%  & 0.0\%& \multirow{4}{*}{$k$-WTA-0.1}& Natural &  99.3\%         & \textbf{62.2\%}     \\
&AT     &  99.2\%   & 95.0\%        & & AT &  99.2\% &  \textbf{96.4\%}  \\  
&TRADES   &  99.2\%   &  96.0\%             &  & TRADES &   99.0\%     &  \textbf{96.9\%} \\
&FAT   &  98.2\%   &  94.7\%              & & FAT &   98.1\%     &  \textbf{96.0\%}  \\
\toprule
\end{tabular}
% \vspace{0.3mm}
% \vspace{-4.5mm}
\end{table}

\subsection{Loss Landscape Visualization}\label{sec:surf}
In addition to the experiments shown in \figref{loss} and \secref{loss_land} of the main text,
we further 
% To help understand how the sparsity ratio $\gamma$ affects \kwta's robustness against 
% gradient-based attacks,
we visualize the loss landscapes of \kwta networks when different sparsity ratios $\gamma$ are used.
The plots are shown in \figref{surf},
produced in the same way as \figref{loss} described in \secref{loss_land}.

As analyzed in \secref{theory}, a larger $\gamma$ tends to 
smooth the loss surface of the \kwta network 
with respect to the input, while a smaller $\gamma$ renders the loss
surface more discontinuous and ``spiky''. 
In addition, adversarial training tends to reduce the range of the loss values---a 
similar phenomenon in ReLU networks has 
already been reported~\cite{madry2017towards,tramer2017ensemble}---but 
that does not mean that the loss surface becomes smoother; the loss surface remains spiky.
\begin{figure*}[h!]
	\centering
	%\vspace{-1mm}
	\includegraphics[width=1.0\textwidth]{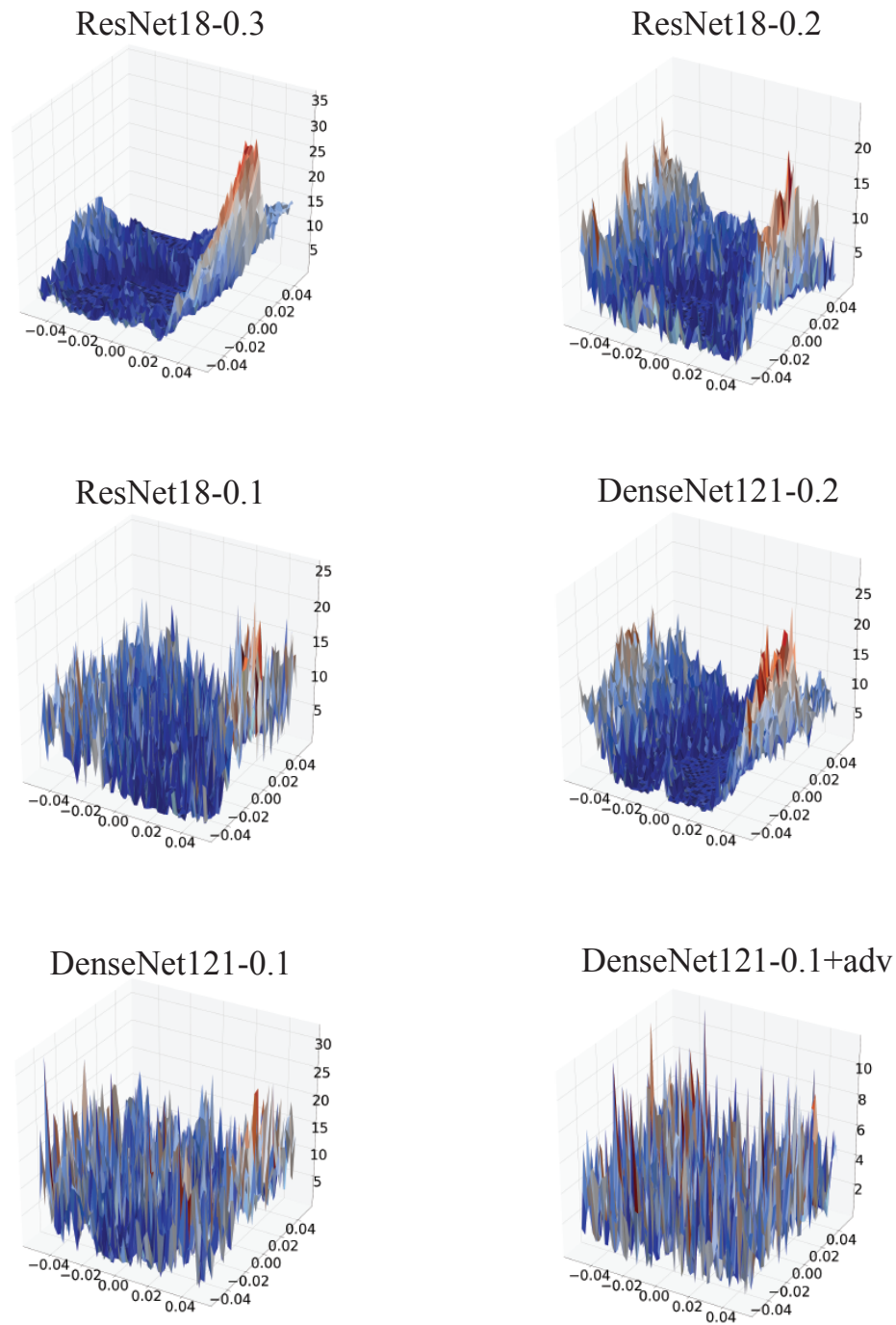}
	\vspace{-5mm}
	\caption{Visualization of ResNet18 and DenseNet121 with different $\gamma$ values.
        The last one (DenseNet121-0.1+adv) is the result using adversarial training.
        The others are optimized using natural (non-adversarial) training.
	}
    \label{fig:surf}
	\vspace{-3mm}
\end{figure*}

\end{document}